\documentclass[twoside]{article}

\usepackage[preprint]{aistats2026}

\usepackage{xr-hyper}
\makeatletter
\newcommand*{\addFileDependency}[1]{% argument=file name and extension
  \typeout{(#1)}
  \@addtofilelist{#1}
  \IfFileExists{#1}{}{\typeout{No file #1.}}
}
\makeatother
\newcommand*{\myexternaldocument}[1]{%
    \externaldocument{#1}%
    \addFileDependency{#1.tex}%
    \addFileDependency{#1.aux}%
}

\usepackage{microtype}
\usepackage{graphicx}
\usepackage{subfigure}
\usepackage{booktabs} 
\usepackage{multirow, threeparttable}
\usepackage{wrapfig}
\usepackage{caption}
\usepackage{hhline}
\usepackage{subcaption}
\usepackage{wrapfig}

\usepackage[round]{natbib}

\usepackage[table, xcdraw, usenames,dvipsnames]{xcolor}
\usepackage{pifont}
\usepackage{hyperref}

\newcommand{\mathup}[1]{\textrm{\upshape{#1}}}

\usepackage{amsmath}
\usepackage{amssymb}
\usepackage{mathtools}
\usepackage{amsthm}

\newcommand*\diff{\mathop{}\!\mathrm{d}}

\usepackage{xspace}
\usepackage[capitalize,noabbrev]{cleveref}

\theoremstyle{plain}
\newtheorem{theorem}{Theorem}[section]

\newtheorem{lemma}[theorem]{Lemma}

\theoremstyle{definition}
\newtheorem{definition}[theorem]{Definition}

\theoremstyle{remark}

\definecolor{lightgray}{HTML}{F5F5F5}

\usepackage{tikz}
\DeclareRobustCommand\circled[1]{\tikz[baseline=(char.base)]{\node[shape=circle,draw=black,minimum size=0.35cm,inner sep=0pt,fill=lightgray] (char) {\fontfamily{phv}\selectfont \scriptsize \textbf{#1}};}}

\newcommand{\framework}{\textsc{\mbox{PrivATE}}\xspace}

\myexternaldocument{appendix}

\begin{document}

\twocolumn[
\aistatstitle{\framework: Differentially Private Confidence Intervals for Average Treatment Effects}

% \aistatsauthor{%
%   Maresa Schröder %\\
%   LMU Munich\\
%   Munich Center for Machine Learning\\
%   % \texttt{maresa.schroeder@lmu.de} \\
%   % examples of more authors
%   \And
%   Justin Hartenstein% \\
%   Stanford University \\
%   %\texttt{jhartens@mail.uni-mannheim.de} \\
%   \And
%   Stefan Feuerriegel% \\
%   LMU Munich\\
%   Munich Center for Machine Learning\\
%   % \texttt{feuerriegel@lmu.de} \\
% }

%\aistatsauthor{Maresa Schröder Maresa Schröder^{1,2,*} \And Justin Hartenstein^3 \AND  Stefan Feuerriegel^{1,2}}

\aistatsauthor{Maresa Schröder \And Justin Hartenstein \And Stefan Feuerriegel}

\aistatsaddress{LMU Munich\\Munich Center for Machine Learning\\\texttt{maresa.schroeder@lmu.de} \And Stanford University\\\texttt{justinha@stanford.edu} \And LMU Munich\\Munich Center for Machine Learning\\\texttt{feuerriegel@lmu.de} }% LMU Munich \AND 2 Munich Center for Machine Learning \AND 3 Stanford University} 

]

% \runningtitle{\framework: Differentially Private Confidence Intervals for Average Treatment Effects}

%\maketitle

\begin{abstract}
The average treatment effect (ATE) is widely used to evaluate the effectiveness of drugs and other medical interventions. In safety-critical applications like medicine, reliable inferences about the ATE typically require valid uncertainty quantification, such as through confidence intervals (CIs). However, estimating treatment effects in these settings often involves sensitive data that must be kept private. In this work, we present \framework, a novel machine learning framework for computing CIs for the ATE under differential privacy. Specifically, we focus on deriving \emph{valid privacy-preserving CIs for the ATE} from observational data. Our \framework framework consists of three steps: (i)~estimating the differentially private ATE through output perturbation; (ii)~estimating the differentially private variance in a doubly robust manner; and (iii)~constructing the CIs while accounting for the uncertainty from both the estimation and privatization steps. Our \framework framework is model agnostic, doubly robust, and ensures valid CIs. We demonstrate the effectiveness of our framework using synthetic and real-world medical datasets. To the best of our knowledge, we are the first to derive a general, doubly robust framework for valid CIs of the ATE under ($\varepsilon,\delta$)-differential privacy.
\end{abstract}

\section{Introduction}
\label{sec:introduction}

% ATE + CIs
Estimating the \emph{average treatment effect (ATE)} from observational data is highly relevant for evaluating the effectiveness of drugs and other medical interventions \citep[e.g.,][]{Ballmann.2015, Buell.2024, Feuerriegel.2024}. To ensure the reliability of ATE estimates, one often needs to ``move beyond the mean'' by accounting for the estimation uncertainty \citep{Heckman.1997, Kneib.2023}. This is typically achieved through \emph{confidence intervals (CIs)}, which provide a range within which the true effect is likely to lie \citep{Neyman.1937}. Medical studies frequently report 95\% CIs to show the precision of the estimated treatment effects. For example, the clinical trial for the Moderna COVID-19 vaccine specified that the effectiveness must meet a predefined threshold based on the 95\% CI, and, ultimately, the trial reported a 95\% CI of 89.3\% to 96.8\% \citep{Baden.2021}.

% DP
Estimating ATEs in medicine typically involves sensitive patient data \citep{Brothers.2015}, which must be kept private. Several regulations have been introduced to enforce data privacy, including the US Health Insurance Portability and Accountability Act (HIPAA) and the EU General Data Protection Regulation (GDPR). One approach to analyse sensitive data is to use privacy mechanisms such as \emph{differential privacy (DP)}, which enables population-level inference while protecting individual information \citep{Dwork.2006}. Due to strong theoretical guarantees, especially ($\varepsilon,\delta$)-DP has received increasing attention \citep[e.g.,][]{Abadi.2016, Bassily.2014}.

In this work, our aim is to \emph{compute differentially private confidence intervals for the ATE}. However, this is \textit{non-trivial}: constructing CIs solely based on the private ATE estimate does \emph{not} suffice to ensure DP of the entire CI, as estimating the variance required for constructing CIs involves additional queries to the underlying data. Thus, it is necessary to develop privacy-preserving mechanisms not only for the ATE estimate itself but also for the variance estimate. In other words, \emph{both components must be computed in a differentially private manner} to ensure privacy.

Yet, this introduces an additional \emph{challenge}: \emph{privatizing the ATE and its variance introduces additional noise}. This added uncertainty needs to be accounted for when computing CIs to ensure \emph{validity}\footnote{A CI is valid if it achieves its nominal coverage rate asymptotically. For example, a 95\% confidence interval is valid if it covers the true parameter approximately 95\% of the time under repeated sampling.}. Failure to do so results in CIs that do not maintain their stated coverage probability. More formally, in estimation under DP, uncertainty arises from two sources: (i)~the sampling error inherent to the data and the uncertainty introduced by the machine learning (ML) model, and (ii)~the noise added to ensure privacy. The second source of uncertainty makes its quantification for DP estimates \emph{inherently different from standard uncertainty quantification}. Therefore, traditional methods to construct CIs fail to account for the latter, resulting in CIs that \underline{\emph{not}} valid and thus give rise to misleading conclusions.

So far, there are only a few works on treatment effect estimation under DP \citep[e.g.,][]{Guha.2024, Lee.2019, Ohnishi.2024}. Existing works mainly focus on estimating causal quantities as point estimates, whereas privacy-preserving CIs are missing. As discussed above, extending existing methods for ATE estimation under DP to construct differentially private CIs is \underline{not} readily feasible. Other studies focus on privacy-preserving CIs \citep[e.g.,][]{Barrientos.2019, Ferrando.2022}, yet outside of treatment effect estimation and, therefore, cannot be directly applied to causal quantities such as the ATE. \citet{Guha.2024} and \citet{Ohnishi.2024} present private CIs for the ATE, but in \emph{different} settings (e.g., by making restrictive assumptions on the data-generating process, by being limited to linear models, and/or by focusing on a different definition of DP). To the best of our knowledge, a method for computing valid CIs for the ATE under ($\varepsilon,\delta$)-DP is missing. 

\begin{figure}[h]
\vspace{-0.4cm}
    \centering
    \includegraphics[width=1\linewidth]{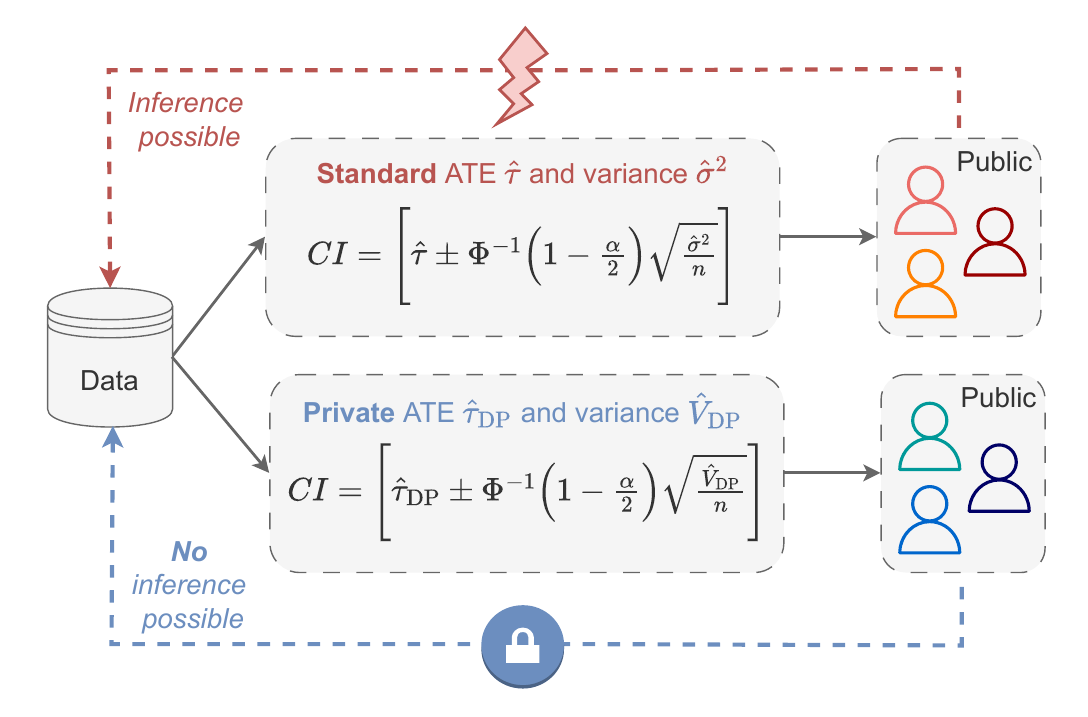}
    \vspace{-0.4cm}
    \caption{\textbf{Standard vs. private CIs for the ATE.} Standard CIs allow for inference about individual data samples, whereas CIs under DP are designed to prevent such inference.}
    \label{fig:privacy}
\end{figure}

Here, we develop \framework, the first framework designed to compute \emph{valid} CIs for the ATE under ($\varepsilon,\delta$)-DP (see Figure~\ref{fig:privacy}). Our framework is built upon the augmented inverse probability weighting (AIPW) estimator, which offers several practical advantages such as being model-agnostic and doubly robust \citep{Robins.1995, Wager.2024}. Notably, model-agnostic means that our framework can be used in conjunction with \emph{any} machine learning model. To obtain the CIs, we further introduce a novel variance estimator for the ATE satisfying DP. This allows us to construct \emph{valid}, model-agnostic CIs even in the presence of privatization noise.

\textbf{Our main contributions:}\footnote{We provide all code and data in our GitHub repository \url{https://github.com/m-schroder/PrivATE}
} (1)~We propose \framework, a novel ML framework to construct differentially private CIs for the ATE. (2)~We prove that our framework guarantees DP and provides valid CIs. (3)~We perform extensive experiments using synthetic and real-world medical data to demonstrate the effectiveness and validity of our framework.

%%%%%%%%%%%%%%%%%%%%%%%%%%%%%%%%%%%%%%%%%%%%%%%%%%%%%%%%%%%%%%%%%%%%%%%%%%%%%%%
% Related work
%%%%%%%%%%%%%%%%%%%%%%%%%%%%%%%%%%%%%%%%%%%%%%%%%%%%%%%%%%%%%%%%%%%%%%%%%%%%%%%

\section{Related work}
\label{sec:related_work}
\vspace{-0.2cm}

We briefly summarize the existing literature relevant to our work, namely (i)~ATE estimation under DP, (ii)~differentially private CI construction, and (iii)~CIs for the ATE. We give an overview in Fig.~\ref{fig:literature}. 

\textbf{Differentially private ATE estimation:} Research on privacy mechanisms for ATE estimation is scarce. \citet{Lee.2019} design a privacy-preserving Horvitz-Thompson type estimator by privatizing the propensity score function. However, due to the employed regularization, the privatized ATE estimate is eventually biased, which prevents the construction of valid CIs. \citet{Guha.2024} provide a DP method for the estimation of ATEs of binary outcomes. Other works have analyzed private treatment effect estimation in other DP regimes like local DP \citep{Agarwal.2024, Ohnishi.2023} or label DP \citep{Javanmard.2024}, or focus on other causal quantities \citep[e.g.,][]{Ohnishi.2024, Schroder.2025, Yao.2024}. However, as explained in our introduction, computing differentially private CIs for the ATE is non-trivial.

\textbf{Differentially private CIs for non-causal quantities:} Several works focus on the construction of differentially private CIs, yet \emph{outside} of causal inference. Here, one stream constructs CIs for quantities estimated through parametric models where the underlying population is known and/or normally distributed \citep[e.g.,][]{Karwa.2018, DOrazio.2015, Ferrando.2022}. Other works have focused on constructing CIs for the parameters of a trained model, e.g., linear regression \citep{Barrientos.2019, Sheffet.2017}. A more general method has been proposed to construct asymptotic confidence intervals for parameters of any model trained by empirical risk minimization \citep{Wang.2019}; however, this method imposes strong assumptions on the convexity and the smoothness of the loss function, making it incompatible with many machine learning models. In sum, the above works focus on non-causal quantities and, therefore, are \emph{\underline{not}} directly applicable to ATE estimation. Therefore, we must first derive an asymptotically normal differentially private ATE estimate before we can then obtain private CIs through our framework. 

\begin{figure}[t]
    \vspace{-.1cm}
    \hspace{-0.45cm} \includegraphics[width=1.1\linewidth]{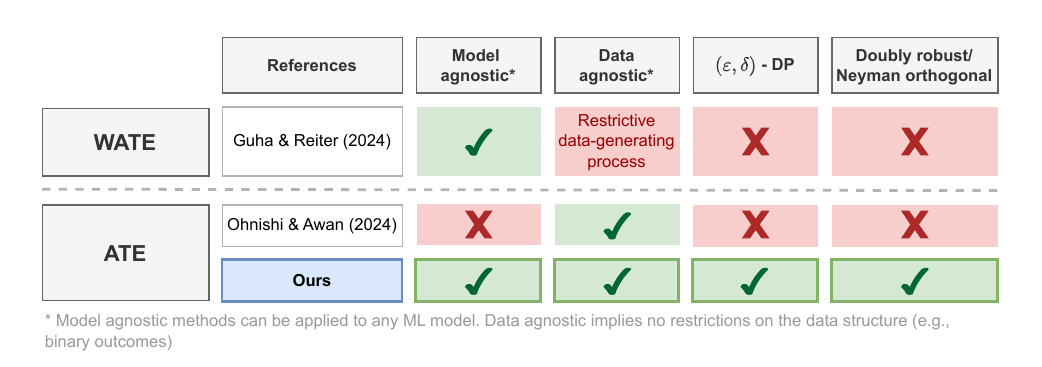}
    \vspace{-0.8cm}
    \caption{Key literature for estimating CIs of the ATE.}
    \label{fig:literature}
    \vspace{-0.5cm}
\end{figure}

\textbf{CIs for the ATE:} Several works construct CIs for the ATE based on $\sqrt{n}$-consistent estimators with asymptotic coverage guarantees \citep{Bang.2005, Hirano.2003}. Some works also focus on finite-sample CIs. However, these require significantly stronger assumptions on the data \citep{Aronow.2021}. Finally, \citet{Wang.2025} uses prediction-powered inference to derive CIs for the ATE from multiple datasets. However, \emph{\underline{no}} method ensures DP.

Literature on \emph{privacy mechanisms} for CIs of the ATE is scarce. We are aware of only two works \citep{Guha.2024,Ohnishi.2024}, yet both works have \emph{clear limitations}. \citet{Guha.2024} focus on the WATE (i.e., the weighted ATE, which, in principle, can be transformed into the ATE), but the method puts restrictive assumptions on the data-generating process, because of which the method is later \emph{\underline{not}} applicable. The approach by \citet{Ohnishi.2024} is limited to a sieve logistic regression and is \emph{\underline{not}} designed to deal with arbitrary ML models. Importantly, both works focus on a \emph{different} privacy definition, namely, pure $\varepsilon$-DP, while we focus on the more common ($\varepsilon,\delta$)-DP.  

\textbf{Research gap}: To our knowledge, we are the first to provide a general machine learning framework to construct valid CIs for the ATE under ($\varepsilon,\delta$)-DP.

%%%%%%%%%%%%%%%%%%%%%%%%%%%%%%%%%%%%%%%%%%%%%%%%%%%%%%%%%%%%%%%%%%%%%%%%%%%%%%%
% Problem formulation
%%%%%%%%%%%%%%%%%%%%%%%%%%%%%%%%%%%%%%%%%%%%%%%%%%%%%%%%%%%%%%%%%%%%%%%%%%%%%%%

\section{Problem formulation}
\label{sec:problem_formulation}

\textbf{Notation:} We denote random variables by capital letters $X$ and their realizations by small letters $x$. We refer to the probability distribution over $X$ by $P_X$, where we omit the subscript if it is apparent from the context. The probability mass function for discrete $X$ is given by $P(x) = P(X=x)$ and the probability density function w.r.t. the Lebesgue measure by $p(x)$. 

\textbf{Setting:} We consider a dataset $D := \{(X_i, A_i, Y_i)\}_{i=1, \ldots , n}$, consisting of observed confounders $X$, a binary treatment $A \in \{0,1\}$, and a bounded outcome $Y \in \mathcal{Y}$, where $Z_i := (X_i, A_i, Y_i) \sim P$ i.i.d., $Z_i \in \mathcal{Z}$, and $\mathcal{X}, \mathcal{Y}$ have bounded domains. We use the potential outcomes framework \citep{Rubin.2005} and denote the potential outcome of intervention $a$ by $Y(a)$. Furthermore, we define the propensity score as $\pi(x) := P(A=1\mid X=x)$ and the outcome function as $\mu(x,a) := \mathbb{E}[Y \mid X=x, A=a]$. 

\textbf{Objective:} We estimate the average treatment effect (ATE) $\tau$ under ($\varepsilon,\delta$)-DP and construct CIs for $\tau:= \mathbb{E}[Y(1) -Y(0)]$ so that the CIs (i)~capture both estimation and privatization uncertainty and (ii)~retain the DP guarantees. We make the standard assumptions in causal inference: positivity, consistency, and unconfoundedness \citep[e.g.,][]{Curth.2021, Rubin.2005}. Then, the ATE is identified as $\tau = \mathbb{E}[Y \mid A=1] - \mathbb{E}[Y \mid A=0]$. We provide an introduction to ATE estimation and an overview of the corresponding estimators in Supplement~\ref{sec:appendix_background}.

In the following, we derive differentially private CIs for the ATE estimated via the AIPW estimator. The AIPW estimator has several advantages \citep{Robins.1994,Wager.2024}: it is  unbiased and asymptotically normally distributed if at least one of the nuisances is correctly specified. Furthermore, the AIPW estimator is \emph{doubly robust} and Neyman-orthogonal, making it insensitive to small errors in the estimated nuisances $\hat{\mu}$ and $\hat{\pi}$ for nuisances converging fast enough to the oracle, i.e., that convergence rates are in $O(n^{-\beta_{\mu}})$ and $O(n^{-\beta_{\pi}})$, respectively, with $\beta_{\mu} + \beta_{\pi} \geq \frac{1}{2}$. We make this standard assumption of sufficiently fast convergence throughout our work \cite[e.g.,][]{Wager.2024, Wang.2025}.

\subsection{Differential privacy}

To ensure the privacy of sensitive information, ($\varepsilon,\delta$)-differential privacy (DP) guarantees that the inclusion or exclusion of any individual data point does not affect the estimated outcome by more than a pre-specified threshold defined by the \emph{privacy budget} $\varepsilon$ and $\delta$ \citep{Dwork.2006,Dwork.2009}. Specifically, the probability density of any outcome $y$ on dataset $D \in \mathcal{Z}^n$ must be \emph{$\varepsilon$-indistinguishable} from the probability density of the same outcome $y$ on a neighboring dataset $D^{'} \in \mathcal{Z}^n$ with a probability of at least $1-\delta$. 

\begin{definition}[Differential privacy \citep{Dwork.2009}]\label{def:diff_privacy}
    Two datasets $D$ and $D^{'}$ are called \emph{neighbors}, denoted as $D \sim D^{'}$, if their Hamming distance equals one, i.e., $d_\mathrm{H} (D, D^{'})=1$. 
    A function $\mathbf{f}_D: D \mapsto \mathbb{R}^d$ trained on dataset $D$ is $(\varepsilon, \delta)$-\emph{differentially private} if, for all neighboring datasets $D$, $D^{'} \in \mathcal{Z}^n$ and all measurable $S \subseteq \mathbb{R}^d$, it holds that
    \begin{align}
        P(\mathbf{f}_D \in S) \leq \exp(\varepsilon) \cdot P(\mathbf{f}_{D^{'}} \in S) + \delta.
    \end{align}
\end{definition}

DP can be achieved by employing the so-called \emph{Gaussian privacy mechanisms} \citep{Dwork.2014}, where the aim is to perturb the prediction by adding appropriately calibrated zero-centered Gaussian noise so that two predictions resulting from neighboring databases cannot be differentiated \citep[e.g.,][]{Chaudhuri.2011,Zhang.2022c}. 

\begin{definition}[Gaussian privacy mechanism \citep{Dwork.2014}]
\label{def:gaussian_mechanism}
    Let $\mathbf{f}: D \mapsto \mathbb{R}^d$ and let $\Delta_2(\mathbf{f}) = \sup _{D \sim D^{'}} ||\mathbf{f}_D - \mathbf{f}_{D^{'}}||$ denote the $l_2$-sensitivity if $\mathbf{f}$. Let  $\mathbf{U} \sim \mathcal{N}(0, \sigma^2)$ for $\sigma \geq \frac{1}{\varepsilon}\sqrt{2\ln{(1.25/\delta)}} \cdot \Delta_2(\mathbf{f})$. Then, $\mathbf{f}^{\mathup{DP}}_D  = \mathbf{f}_D + \mathbf{U}$ is $(\varepsilon, \delta)$-differentially private.
\end{definition}

\subsection{Problem statement}

Uncertainty quantification for the ATE is commonly presented in terms of confidence intervals.\footnote{For an asymptotically consistent ATE estimate $\hat{\tau}$ and its estimated variance $\hat{\sigma}^2$, the CI for a confidence level $(1-\alpha)$, $\alpha \in (0,1)$, is given by
$
\mathrm{CI} = \Big[ \hat{\tau} \pm \Phi^{-1}\Big(1-\frac{\alpha}{2}\Big) \sqrt{\frac{\hat{\sigma}^2}{n}}\Big],
$
where $\Phi^{-1}\Big(1-\frac{\alpha}{2}\Big)$ denotes the critical value given by the respective quantile of the standard normal distribution.} In our work, we aim at \emph{differentially private uncertainty quantification} of the ATE. Specifically, we aim to derive an $(\varepsilon, \delta)$-differentially private estimate $\hat{\tau}_{\mathup{DP}}$ and the corresponding differentially private CI
\begin{align}
\label{eq:CIs}
    \mathrm{CI}_{\mathup{DP}} := \Bigg[ \hat{\tau}_{\mathup{DP}} \pm \Phi^{-1}\Big(1-\frac{\alpha}{2}\Big) \sqrt{\frac{\hat{V}_{\mathup{DP}}}{n}}\Bigg]
\end{align}
for a confidence level $(1-\alpha)$, where $\hat{V}_{\mathup{DP}}$ denotes the \emph{augmented differentially private variance estimate} capturing the variation stemming from \emph{both} the estimation and the privatization process. Importantly, estimating $\hat{V}_{\mathup{DP}}$ is \emph{non-trivial}: setting $\hat{V}_{\mathup{DP}} = \hat{\sigma}^2_{\mathup{DP}}$ for a privatized variance estimate of the ATE would lead to \emph{undercoverage} of the final private CI (see Section~\ref{sec:method}). Therefore, we must carefully derive the augmented variance estimate $\hat{V}_{\mathup{DP}}$ to guarantee valid CIs.

%%%%%%%%%%%%%%%%%%%%%%%%%%%%%%%%%%%%%%%%%%%%%%%%%%%%%%%%%%%%%%%%%%%%%%%%%%%%%%%
% Method
%%%%%%%%%%%%%%%%%%%%%%%%%%%%%%%%%%%%%%%%%%%%%%%%%%%%%%%%%%%%%%%%%%%%%%%%%%%%%%%

\section{Our \framework framework}
\label{sec:method}

\subsection{Overview}

\begin{figure}[h]
    \centering
    \includegraphics[width=1\linewidth]{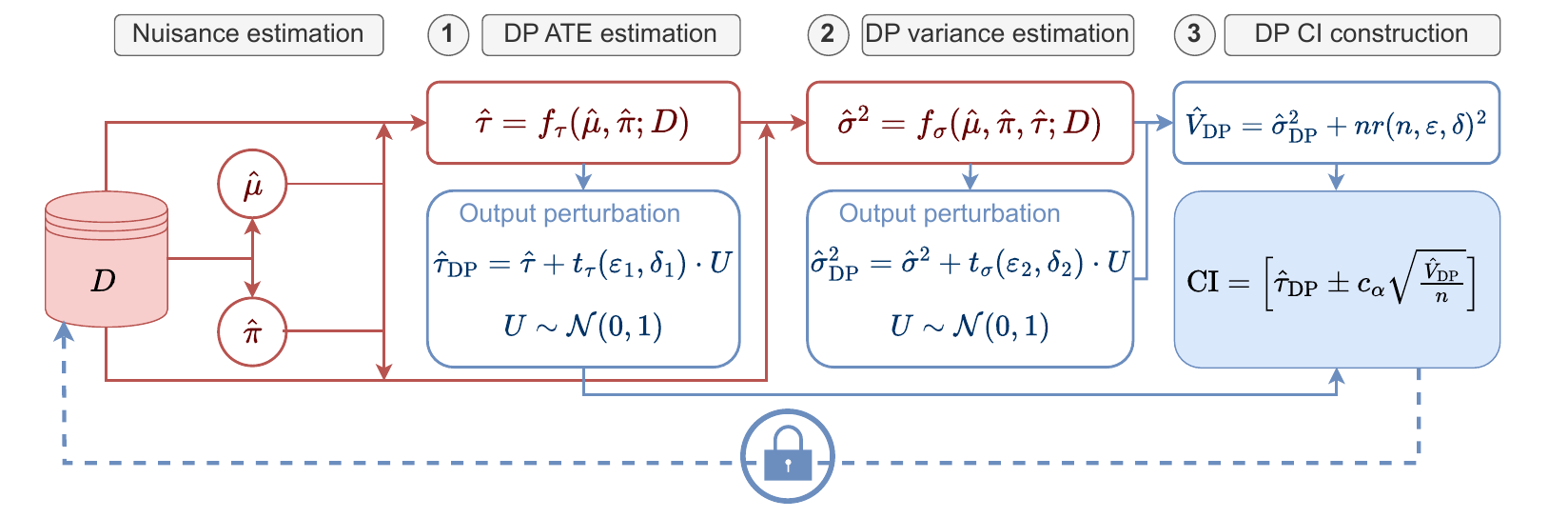}
    \vspace{-0.5cm}
    \caption{\textbf{Our proposed \framework framework} for constructing $(\varepsilon, \delta)$-differentially private CIs for the ATE.}
    \label{fig:framework}
    \vspace{-0.4cm}
\end{figure}

Our framework for constructing $(\varepsilon, \delta)$-differentially private CIs for the ATE proceeds in three steps: \circled{1}~\textbf{ATE privatization $\hat{\tau}_{\mathup{DP}}$:} We estimate the ATE $\hat{\tau}$ through doubly robust estimation and guarantee DP via output perturbation ($\rightarrow$ our Theorem~\ref{thm:ate_privatization}). \circled{2}~\textbf{Differentially private variance estimation $\hat{\sigma}_{\mathup{DP}}$:} To be able to construct valid CIs under DP, we additionally obtain an estimate of the variance of $\hat{\tau}$ satisfying DP guarantees ($\rightarrow$ our Theorem~\ref{thm:variance_privatization}). \circled{3}~\textbf{Constructing CIs for the ATE under DP:} We finally employ $\hat{\tau}_{\mathup{DP}}$ and $\hat{\sigma}_{\mathup{DP}}$ to construct differentially private CIs for the ATE. Notably, the resulting CIs capture the uncertainty stemming from the non-private estimation in Step~\circled{1} as well as the additional uncertainty arising due to the privatization in Step~\circled{2} ($\rightarrow$ our Theorem~\ref{thm:intervals}). We show our framework in Figure~\ref{fig:framework}.

Of note, our \framework framework is highly flexible. We do not make any assumptions about the parametric form of the model or the data distribution. Our framework thus is \emph{model-agnostic} and can be seamlessly instantiated with any machine learning model. Furthermore, our framework applies to observational and experimental data as we do \emph{not} assume the propensity function to be known.

\subsection{Mathematical preliminaries}
\label{sec:mathematical_background}

We focus on the AIPW estimator in our derivations as it (i)~consistently estimates the ATE and (ii)~is asymptotically normal, allowing us to derive valid CIs. The AIPW estimator is given by
\begin{align}
    \hat{\tau} = \frac{1}{n}\sum_{i=1}^n \Bigg[\hat{\mu}(X,1) - \hat{\mu}(X,0)
     + \frac{Y_i - \hat{\mu}(X,1)}{\hat{\pi}(X_i)}A_i \\
     - \frac{Y_i - \hat{\mu}(X,0)}{1 - \hat{\pi}(X_i)} (1- A_i)\Bigg].
\end{align}
Our work exploits a connection between model robustness and privacy in terms of the \emph{influence function} of the underlying estimator following former work in this field \citep[e.g.,][]{Dwork.2009, Schroder.2025}.

\begin{definition}\label{def:IF}
    Let $T$ be a functional of a distribution with $T(P)$ the parameter of interest. The \emph{influence function} (IF) of $T$ at data point $z$ under distribution $P$ is defined as
    \vspace{-0.1cm}
    \begin{align}
        \mathrm{IF}(z,T;P) := \lim_{t \mapsto 0} \frac{T((1-t)P + t\delta_z) - T(P)}{t},
    \end{align}
    where $\delta_z$ denotes the Dirac-delta functional at $z$. The \emph{gross-error sensitivity} of $T$ is defined as 
    \begin{align}
        \gamma(T,P) := \sup_{z \in \mathcal{Z}} \lVert \mathrm{IF}(z,T;P) \rVert.
    \end{align}
\end{definition}

\begin{lemma}\label{lem:DR-IF}
    We make the following assumptions on the nuisance functions: The nuisance functions are (i)~bounded, (ii)~estimated at rates $n^{-\beta_{\mu}}$ and $n^{-\beta_{\pi}}$ with $\beta_{\mu} + \beta_{\pi} \geq \frac{1}{2}$, and (iii)~in a local neighborhood of the true nuisance functions, i.e., there exists $\lambda_n$ decreasing in the sample size $n$, s.t. $\rVert \hat{\eta}-\eta_0\lVert{\infty} \leq \lambda_n$.
    Then, the IF of the AIPW learner is dominated by the influence function of the final mean estimation in the second step. 
    The influence stemming from training the nuisance estimators $\hat{\pi}$ and $\hat{\mu}$ is negligible in the privatization step.
\end{lemma}
\vspace{-0.5cm}
\begin{proof}
    We prove Lemma~\ref{lem:DR-IF} in Supplement~\ref{sec:appendix_proofs}.
\end{proof}

\textbf{Remark:}
Assumptions (i) states a weak regularity condition, restricting the nuisance predictions to fall into the bounded domains $\mathcal{A}, \mathcal{Y}$. Assumption (ii) postulates the standard convergence rate requirement for Neyman-orthogonality introduced in Section~\ref{sec:problem_formulation}. Finally, assumption (iii) requires that the nuisances are estimated sufficiently well. Together with the Neyman-orthogonality, it renders the influence stemming from the first stage estimation negligible for privatization. 

Lemma~\ref{lem:DR-IF} is of great benefit: Due to the construction of the mean estimation, the final AIPW estimator is \emph{not} sensitive to small perturbations in the nuisance functions. 
This allows us to employ the complete data for the second-stage averaging, as we do not need to privatize the nuisance estimates. Together with its consistency and asymptotic normality, the AIPW estimator is a perfect underlying method for providing differentially private confidence intervals of the ATE. 

\textbf{Remark:}
To avoid making the local neighborhood assumption, one can alternatively privatize the nuisances by a DP method of choice. The privatization neither affects the efficiency of the estimator \citep{Schroder.2025} nor the validity of the CIs. Specifically, we split the dataset $D$ into disjoint subsets $D_1$ and $D_2$. On $D_1$, we estimate $\hat{\pi}, \hat{\mu}$ under $(\varepsilon_1/2, \delta_1/2)$-DP, which are then employed to estimate $\hat{\tau}$ on $D_2$. Note that data splitting is common to prevent overfitting of the estimators \citep[e.g.,][]{Foster.2019, Wager.2024}.

\subsection{Framework}

\subsection*{\circled{1} Differentially private ATE estimation}
\label{sec:ate_privatization}

In the first step, we privatize the original ATE estimate $\hat{\tau}$. Directly employing Definition~\ref{def:gaussian_mechanism} requires calculating the model sensitivity, which can be computationally expensive or even infeasible. We instead exploit the fact that a suitably scaled gross-error sensitivity $\gamma$ can upper bound the global sensitivity \citep{AvellaMedina.2021}. Therefore, we first derive the IF of the AIPW estimator so that we are then able to privatize $\hat{\tau}$ based on $\gamma$.

For the AIPW estimator, the \emph{influence function score} is given by \citep{Robins.1995, Wager.2024}
\footnotesize
\begin{align}
    \Gamma_{\hat{\eta}}(z)
    = &\left( \frac{a_i}{\hat{\pi}(x_i)} - \frac{1-a_i}{1-\hat{\pi}(x_i)} \right) y_i \\
     &- \frac{\left(1-\hat{\pi}(x_i)\right)\hat{\mu}_1(x_i) + \hat{\pi}(x_i)\hat{\mu}_0(x_i)}{\hat{\pi}(x) \left(1-\hat{\pi}(x_i)\right)} (a_i - \hat{\pi}(x_i)).
     \nonumber
\end{align}
\normalsize
Since the AIPW estimator can equivalently be written as $\hat{\tau} = \sum_{i=1}^n \Gamma_{\hat{\eta}}(z_i)$,
the IF of the AIPW estimator directly follows via\footnote{For an introduction to the construction of IFs for common statistics such as the mean, we refer to \citep{Hines.2022}.}
\begin{align}
    \mathrm{IF}^{\mathup{AIPW}}(z,\tau_0;P) = \Gamma_{\hat{\eta}}(z) - \hat{\tau}.
\end{align}

\begin{theorem}[ATE privatization]
\label{thm:ate_privatization}
    Let $z:= (a,x,y)$ define a data sample following the joint distribution $\mathcal{Z}$, and let $\hat{\eta} = (\hat{\mu},\hat{\pi})$ be the nuisance functions estimated at rates $n^{-\beta_{\mu}}$ and $n^{-\beta_{\pi}}$ with $\beta_{\mu} + \beta_{\pi} \geq \frac{1}{2}$, and $\hat{\tau}$ the non-private ATE estimate. Furthermore, let $D$ be the training dataset with $\vert D \vert$ = $n$. For $U \sim \mathcal{N}(0, 1)$, $(\varepsilon, \delta)$-DP is fulfilled by
    \footnotesize
    \vspace{-0.2cm}
    \begin{align} 
        \hat{\tau}_\mathup{DP} := \hat{\tau} + \sup_{z \in \mathcal{Z}} \big \lVert
        \Gamma_{\hat{\eta}}(z) - \hat{\tau} \big \rVert \cdot \frac{5\sqrt{2\ln(n)\ln{(2/\delta)}}}{\varepsilon n} \cdot U,
    \end{align}
    \normalsize
\end{theorem}
\vspace{-0.4cm}
\begin{proof}
    We prove Theorem~\ref{thm:ate_privatization} in Supplement~\ref{sec:appendix_proofs}.
\end{proof}

\subsection*{\circled{2} Differentially private variance estimation}
\label{sec:variance_privatization}

Solely constructing CIs based on the private ATE estimate is \emph{not} sufficient to guarantee DP, as estimating the variance necessary for constructing the CIs requires additional queries to the underlying data. Therefore, we must now derive a differentially private estimate of the variance, $\hat{\sigma}^2_\mathup{DP}$, for $\hat{\tau}_\mathup{DP}$.

The variance estimate depends on the estimated nuisance functions. Standard variance estimators are, in general, \emph{only} consistent if \emph{both} nuisance functions are correctly specified \citep{Gruber.2012}. Thus, straightforward estimators, such as the influence function-based estimator ~\citep{Deville.1999}, are not doubly robust and, therefore, can easily lead to \emph{undercoverage} of the CIs in practice, which is undesirable in safety-critical applications. As a remedy, we estimate $\hat{\sigma}_{\mathup{AIPW}}^2$ by a \emph{doubly robust variance estimator}.

It is common to obtain such an estimator through nonparametric bootstrapping \citep{Funk.2011, ShookSa.2024}. However, bootstrapping requires many queries to the data, which is disadvantageous for estimation under DP guarantees, as each query requires a separate privacy budget. 

Instead, in our work, we estimate $\hat{\sigma}_{\mathup{AIPW}}^2$ via the \emph{empirical sandwich estimator} \citep{Huber.1967}. This has the following benefits: (i)~the empirical sandwich estimator is doubly robust, and  (ii)~it consistently estimates the variance of the AIPW estimator \citep{ShookSa.2024}. Importantly, for the AIPW estimator, the sandwich estimator can be directly represented by the IF (see Supplement~\ref{sec:appendix_background})
\begin{align}
    \hat{\sigma}_{\mathup{AIPW}}^2 = \frac{1}{n}\sum_{i=1}^n  \mathrm{IF}^{\mathup{AIPW}}(z,\tau_0;P)^2.
\end{align}    
To derive a differentially private version of $\hat{\sigma}_{\mathup{AIPW}}^2$, we now follow the same strategy as in Step \circled{1}. We first derive the IF of  $\hat{\sigma}_{\mathup{AIPW}}^2$ and then perform output perturbation employing the gross-error sensitivity to upper-bound the smooth sensitivity of $\hat{\sigma}_{\mathup{AIPW}}^2$. 

\begin{lemma}
\label{lem:IF_variance}
    The influence function of $\hat{\sigma}_{\mathup{AIPW}}^2$ is given by
    \begin{align}
        \mathrm{IF}^{\sigma}(z, \sigma_0, P) = (\Gamma_{\hat{\eta}}(z) - \hat{\tau})^2 - \hat{\sigma}_{\mathup{AIPW}}^2.
    \end{align}
\end{lemma}
%\vspace{-0.4cm}
\begin{proof}
    We prove Theorem~\ref{lem:IF_variance} in Supplement~\ref{sec:appendix_proofs}.
\end{proof}

Additionally, we observe that the differentially private variance is naturally positive. To impose this restriction during the privatization process, we employ a \emph{truncated privacy mechanism}: 
\emph{Let $\mathbf{f}: D \mapsto \mathbb{R}^d$ and let $\Delta_2(\mathbf{f}) = \sup _{D \sim D^{'}} ||\mathbf{f}_D - \mathbf{f}_{D^{'}}||$ denote the $l_2$-sensitivity of $\mathbf{f}$. Furthermore, let  $\mathbf{U} \sim \mathcal{N}(0, \sigma^2)$ for $\sigma \geq \frac{1}{\varepsilon}\sqrt{2\ln{(1.25/\delta)}} \cdot \Delta_2(\mathbf{f})$. Then, 
\begin{equation}
\mathbf{f}^{\mathup{DP}}_D  = \max \{ 0, \mathbf{f}_D + \mathbf{U} \}
\end{equation}
is $(\varepsilon, \delta)$-differentially private.}
We prove the statement in the proof of Theorem~\ref{thm:variance_privatization}.

\begin{theorem}[Differentially private variance estimation]
\label{thm:variance_privatization}
    Let $z:= (a,x,y)$ define a data sample following the joint distribution $\mathcal{Z}$, and let $\hat{\eta} = (\hat{\mu},\hat{\pi})$ be the nuisance functions estimated at rates $n^{-\beta_{\mu}}$ and $n^{-\beta_{\pi}}$ with $\beta_{\mu} + \beta_{\pi} \geq \frac{1}{2}$. Furthermore, let $D$ be the training dataset with $\vert D \vert$ = $n$.  For $U \sim \mathcal{N}(0, 1)$, we define
    \footnotesize
    \begin{align}
        \hat{\sigma}^2_\mathup{DP} := &\max \bigg\{0, \hat{\sigma}_{\mathup{AIPW}}^2 \\ &+ \sup_{z \in \mathcal{Z}} \big \lVert
        \mathrm{IF}^{\sigma}(z, \sigma_0, P) \big \rVert \cdot \frac{5\sqrt{2\ln(n)\ln{(2/\delta)}}}{\varepsilon n} \cdot U \bigg\}.
    \end{align}
    \normalsize
    \noindent
    Then, $\hat{\sigma}^2_\mathup{DP}$ is $(\varepsilon, \delta)$-differentially private. 
\end{theorem}
\begin{proof}
    We prove Theorem~\ref{thm:variance_privatization} in Supplement~\ref{sec:appendix_proofs}.
\end{proof}

\subsection*{\circled{3} Constructing CIs for the ATE under DP}
\label{sec:intervals}

Simply privatizing the variance to ensure DP is not sufficient to construct valid CIs in general:
\emph{only} in the \emph{asymptotic regime} the intervals are valid, as no privatization is needed there and, as a result, the privatization noise equals zero. However, in finite samples,  the privatized variance \emph{may underestimate the true variance}, leading to undercoverage in the CIs.

Observe that the AIPW estimator $\hat{\tau}$ is a consistent estimator for the true ATE $\tau_0$ and asymptotically normally distributed with variance $\sigma^2_{\mathup{AIPW}}$ \citep[e.g.,][]{Wager.2024}, i.e.,
\begin{align}
    \sqrt{n}\left(\hat{\tau} - \tau_0\right) \rightarrow \mathcal{N} \left(0, \sigma^2_{\mathup{AIPW}}\right) .
\end{align}
For $\hat{\tau}_{\mathup{DP}}$, by Slutsky's theorem, the same asymptotic guarantee holds. This implies that the privatization asymptotically does not have an effect, since, for infinite data, no privatization is needed. 

In finite settings, however, the privatization introduces additional noise on $\hat{\tau}$ (and thus variance), which we account for in the following when constructing our CIs. This is in line with prior literature postulating that DP inference should be conservative to ensure reliability even in small sample sizes \citep[e.g.,][]{Karwa.2017}. Observe that $(\hat{\tau}_{\mathup{DP}} - \hat{\tau})=:B$ is itself a random variable (for fixed $\hat{\tau}$) following distribution $\mathcal{N} \left(0, r(\varepsilon, \delta, n)^2\right)$, where 
\footnotesize
\begin{align}
    r(\varepsilon, \delta, n) := \sup_{z \in \mathcal{Z}} \big \lVert
        \Gamma_{\hat{\eta}}(z) - \hat{\tau} \big \rVert \cdot \frac{5\sqrt{2\ln(n)\ln{(2/\delta)}}}{\varepsilon n}
\end{align}
\normalsize
and therefore
\begin{align} \label{eq:splitting}
    \sqrt{n}\left(\hat{\tau}_{\mathup{DP}} - \tau_0\right) = \sqrt{n}B + \sqrt{n}\left(\hat{\tau} - \tau_0\right).
\end{align}
To incorporate the \emph{finite sample variance} stemming from the privatization, we see from \eqref{eq:splitting} that, for the variance $V$, it must hold that: $V = \sigma_{\mathup{AIPW}}^2 + n r(\varepsilon, \delta, n)^2$.

To construct the CIs, we further make use of two important properties of DP: the (i)~\emph{post-processing} and the (ii)~\emph{sequential decomposition} property (see Supplement~\ref{sec:appendix_background} for details). The post-processing property states that differentially private estimates cannot be de-privatized by further processing of the estimate. The sequential decomposition property states that the privacy budgets $\varepsilon_i$ of multiple sequential estimation tasks $i$ on the same data add up when releasing a joint outcome. For our task, the property~(i) justifies that constructing the intervals from differentially private ATE and variance estimates does not require further privatization. Further, property~(ii) establishes that we must split the overall privacy budget into two smaller budgets, $\varepsilon_1$ and $\varepsilon_2$, attributed to ATE and variance estimation.

With the properties above and Steps \circled{1} and \circled{2}, we now state the main contribution of our work. 

\begin{theorem}[Valid CIs for the ATE under DP]
\label{thm:intervals}
Let $z:= (a,x,y)$ define a data sample following the joint distribution $\mathcal{Z}$, and let $\hat{\eta} = (\hat{\mu},\hat{\pi})$ be the estimated nuisance functions estimated at rates $n^{-\beta_{\mu}}$ and $n^{-\beta_{\pi}}$ with $\beta_{\mu} + \beta_{\pi} \geq \frac{1}{2}$. We define the privacy budget $(\varepsilon, \delta)$ and the sub-budgets $\varepsilon_1, \varepsilon_2, \delta_1, \delta_2$, for ATE and variance estimation under DP such that $\varepsilon_1 + \varepsilon_2 = \varepsilon$ and $\delta_1 + \delta_2 = \delta$.
Furthermore, let $D$ be the training dataset with $\vert D \vert$ = $n$, and let $\hat{\tau}_{\mathup{DP}}$ denote the $(\varepsilon_1, \delta_1)$-differentially private ATE estimate and $\hat{\sigma}_{\mathup{DP}}$ denote the $(\varepsilon_2, \delta_2)$-differentially private variance estimate. Then, the valid and $(\varepsilon, \delta)$-differentially private CI is given by
\vspace{-0.2cm}
\begin{align}
    \mathrm{CI}_{\mathup{DP}} := \Bigg[ \hat{\tau}^{\mathup{DP}} \pm \Phi^{-1}\Big(1-\frac{\alpha}{2}\Big) \sqrt{\frac{\hat{V}_{\mathup{DP}}}{n}}\Bigg],
\end{align}
where $\Phi^{-1}(1-\alpha/2)$ denotes the respective quantile of the standard normal distribution, $\gamma_{\tau}$ is the gross-error sensitivity of the ATE estimation, and
\begin{align}
    \hat{V}_{\mathup{DP}} =  \hat{\sigma}^2_{\mathup{DP}}+ \gamma_{\tau}^2 \cdot \frac{50\ln(n)}{n\varepsilon_1^{2}}\ln{\Big(\frac{2}{\delta_1}\Big)}.
\end{align}
\end{theorem}
\vspace{-0.3cm}
\begin{proof}
    We prove Theorem~\ref{thm:intervals} in Supplement~\ref{sec:appendix_proofs}.
\end{proof}

%%%%%%%%%%%%%%%%%%%%%%%%%%%%%%%%%%%%%%%%%%%%%%%%%%%%%%%%%%%%%%%%%%%%%%%%%%%%%%%
% Experiments
%%%%%%%%%%%%%%%%%%%%%%%%%%%%%%%%%%%%%%%%%%%%%%%%%%%%%%%%%%%%%%%%%%%%%%%%%%%%%%%

\section{Experiments}
\label{sec:experiments}
\vspace{-0.2cm}

\textbf{Baselines:} There are \emph{\underline{no}} other methods for constructing CIs for the ATE under $(\varepsilon, \delta)$-DP. Therefore, we cannot compare our \framework against existing baselines, but instead, evaluate our framework in terms of empirical coverage (see below). We further report standard non-private CIs and CIs computed via a \emph{non-private} variance $\hat{\sigma}^2_{\mathrm{AIPW}}$ instead of $\hat{V}_{\mathup{DP}}$ together with $\hat{\tau}_{\mathrm{DP}}$. We refer to this as na{\"i}ve CIs in the following. We expect similar coverage for the non-private and our intervals, but undercoverage for the na{\"i}ve CIs. However, the standard and na\"ive CIs are \emph{not} private.

\vspace{-0.1cm}
\begin{table}[t]
\centering
\tiny
\setlength{\tabcolsep}{3pt}
\begin{tabular}{c p{0.8cm} c p{1.7cm} p{1.7cm} p{1.7cm}}
%\begin{tabular}{clcccc}
\toprule
& \multirow{2}{*}{Base} & \multirow{2}{*}{Conf.} & \multicolumn{3}{c}{Coverage} \\
\cmidrule(lr){4-6}
& learner & ($1-\alpha$) & Standard CIs & Na\"ive CIs & \framework (ours)\\
\midrule
\multirow{6}{*}{1}
& Kernel & 0.80 & \textcolor{ForestGreen}{0.792} $\pm$ 0.018 & \textcolor{BrickRed}{0.008} $\pm$ 0.004 & \textcolor{ForestGreen}{0.784} $\pm$ 0.018 \\
& NN     & 0.80 & \textcolor{BrickRed}{0.762} $\pm$ 0.019 & \textcolor{BrickRed}{0.024} $\pm$ 0.007 & \textcolor{ForestGreen}{0.824} $\pm$ 0.017 \\
\cmidrule(lr){2-6}
& Kernel & 0.90 & \textcolor{ForestGreen}{0.902} $\pm$ 0.013 & \textcolor{BrickRed}{0.036} $\pm$ 0.008 & \textcolor{ForestGreen}{0.912} $\pm$ 0.013 \\
& NN     & 0.90 & \textcolor{BrickRed}{0.868} $\pm$ 0.015 & \textcolor{BrickRed}{0.028} $\pm$ 0.007 & \textcolor{ForestGreen}{0.894} $\pm$ 0.014 \\
\cmidrule(lr){2-6}
& Kernel & 0.95 & \textcolor{ForestGreen}{0.960} $\pm$ 0.009 & \textcolor{BrickRed}{0.038} $\pm$ 0.009 & \textcolor{ForestGreen}{0.964} $\pm$ 0.008 \\
& NN     & 0.95 & \textcolor{ForestGreen}{0.950} $\pm$ 0.010 & \textcolor{BrickRed}{0.020} $\pm$ 0.006 & \textcolor{ForestGreen}{0.962} $\pm$ 0.009 \\
\midrule
\multirow{6}{*}{2} 
& Kernel & 0.80 & \textcolor{BrickRed}{0.770} $\pm$ 0.019 & \textcolor{BrickRed}{0.004} $\pm$ 0.001 & \textcolor{ForestGreen}{0.848} $\pm$ 0.016 \\
& NN     & 0.80 & \textcolor{ForestGreen}{0.800} $\pm$ 0.018 & \textcolor{BrickRed}{0.012} $\pm$ 0.005 & \textcolor{ForestGreen}{0.832} $\pm$ 0.017 \\
\cmidrule(lr){2-6}
& Kernel & 0.90 & \textcolor{ForestGreen}{0.904} $\pm$ 0.013 & \textcolor{BrickRed}{0.008} $\pm$ 0.004 & \textcolor{ForestGreen}{0.910} $\pm$ 0.013 \\
& NN     & 0.90 & \textcolor{ForestGreen}{0.912} $\pm$ 0.013 & \textcolor{BrickRed}{0.008} $\pm$ 0.004 & \textcolor{ForestGreen}{0.904} $\pm$ 0.013 \\
\cmidrule(lr){2-6}
& Kernel & 0.95 & \textcolor{ForestGreen}{0.950} $\pm$ 0.010 & \textcolor{BrickRed}{0.010} $\pm$ 0.004 & \textcolor{ForestGreen}{0.958} $\pm$ 0.009 \\
& NN     & 0.95 & \textcolor{ForestGreen}{0.952} $\pm$ 0.010 & \textcolor{BrickRed}{0.016} $\pm$ 0.006 & \textcolor{ForestGreen}{0.954} $\pm$ 0.009 \\
\bottomrule
\end{tabular}
\caption{\textbf{Empirical coverage.} Comparison of \framework against standard and na{\"i}ve CIs based on Kernel and NN as base learners over 500 runs for $\varepsilon = 0.5$ and $\delta = 10^{-5}$. We report the mean and standard deviation across the runs. To retain as much of the prediction power of $\hat{\tau}^{\text{DP}}$ as possible, we set $\varepsilon_1=0.9\varepsilon$, $\delta_1=0.9\delta$. In \textcolor{ForestGreen}{green}: valid CIs cover the desired confidence level $(1-\alpha$) and are thus faithful. $\Rightarrow$\,\emph{As expected, we observe similar coverage for the standard CIs and \framework. The na{\"i}ve CIs fail to fulfill any coverage guarantees, underlining the need for accounting for the finite sample variance from the privatization of $\hat{\tau}$.}}
\label{tab:coverage-analysis}
\vspace{-0.5cm}
\end{table}

\textbf{Implementation:} We evaluate two versions of \framework with different base learners for nuisance estimation: (i)~logistic regression for propensity estimation and kernel ridge regression for the outcome function (\textbf{Kernel}),(ii)~neural networks with $\tanh$ regularization 
for both nuisances (\textbf{NN}) optimized through stochastic gradient descent.

\textbf{Performance metrics: } We evaluate \framework based on the notion of \emph{faithfulness} \citep[e.g.,][]{Schroder.2024}, i.e., we compute the \emph{empirical coverage} of the CIs and assess whether it surpasses the $1-\alpha$ threshold.

\subsection{Evaluation on synthetic datasets}

Evaluating methods in causal ML is difficult, as counterfactual outcomes are never observed in real-world data. Therefore, it is standard to evaluate the methods on synthetic data \citep[e.g.,][]{Curth.2021}. We use two datasets adapted from \citet{Oprescu.2019}: $\bullet$\,\textbf{Dataset~1} consists of two covariates, both of which influence treatment assignment and outcomes. $\bullet$\,\textbf{Dataset~2} is a more complex dataset consisting of 25 covariates, of which a subset acts as confounders. For details, see Supplement~\ref{sec:appendix_experiments}.

\textbf{Results:} We show that our \framework gives CIs that are \emph{faithful}. The results for both datasets are reported in Table~\ref{tab:coverage-analysis} for confidence levels $(1-\alpha) \in \{0.8,0.9,0.95\}$. We also report the coverage of the na\"ive CIs with \emph{non-private} variance and the standard non-private CIs. Recall that our aim is not to benchmark the methods but rather to show that our framework yields faithful CIs. $\Rightarrow$ We observe that \framework achieves the desired coverage guarantees, which confirms that the CIs are faithful.

\textbf{Performance across varying privacy budgets:} We vary the privacy budget $\varepsilon$ in Figure~\ref{fig:facet}. We observe that \framework behaves as expected: for increasing privacy budgets, the intervals become tighter, eventually converging to the non-private estimate. We note that privacy inevitably induces a trade-off regarding the width of the interval, which we discuss in Supplement~\ref{sec:appendix_tradeoff}.

\textbf{Performance across varying sample sizes:} We further analyze the performance of \framework across different sample sizes $n$ (see Fig.~\ref{fig:facet}). The results are as expected and confirm our theoretical insights: For larger $n$, the CIs become tighter, converging to the asymptotic solution. Additionally, we observe that \framework is not sensitive to the choice of the base learner: both versions with (i)~kernel ridge regression (=Kernel) and  (ii)~neural network (=NN) lead to similar results. This again demonstrates the flexibility of our framework. 

\begin{figure}[h]
    \centering
   % \vspace{-.5cm}
    \includegraphics[width=\linewidth]{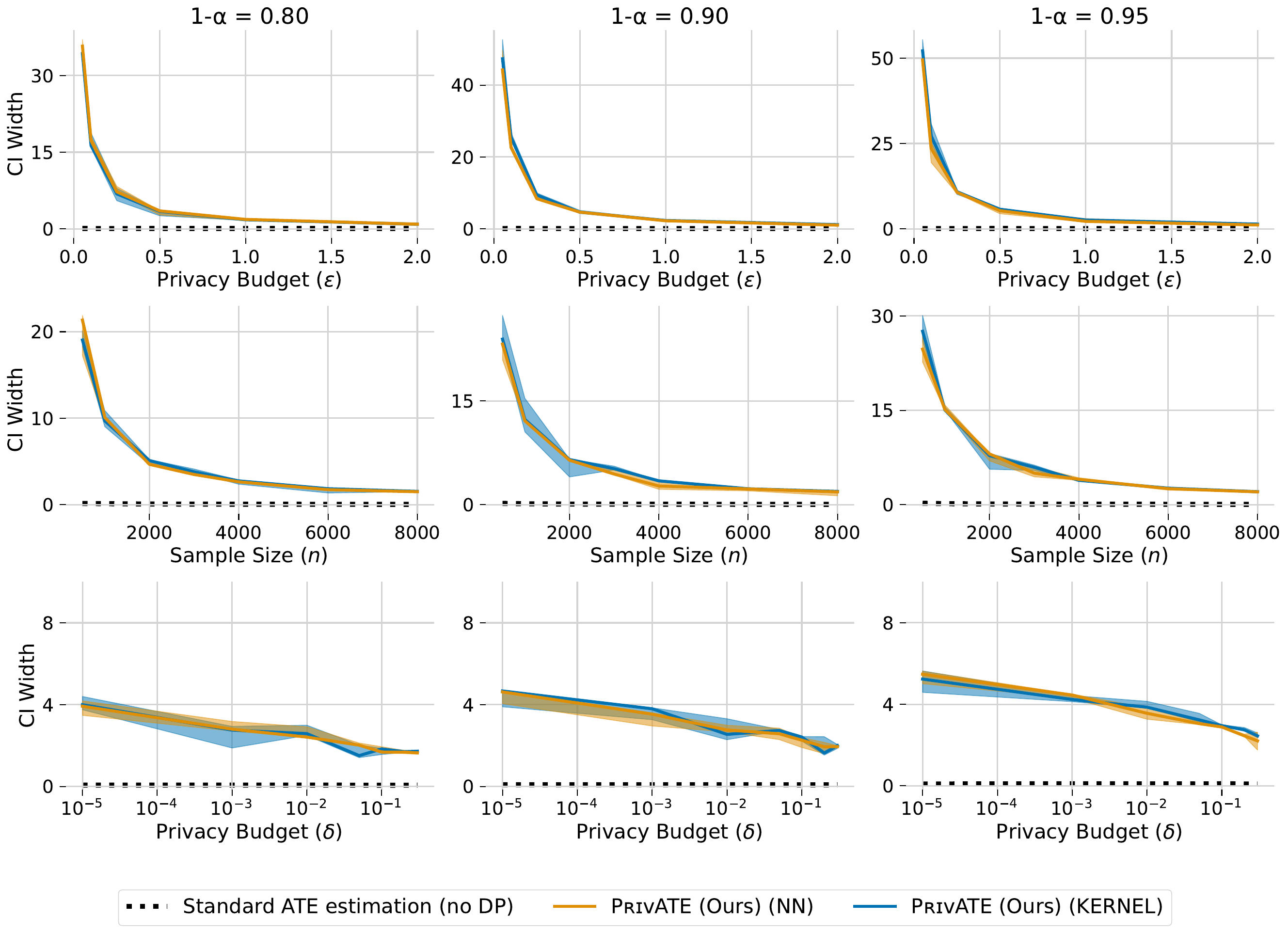}
    \caption{\textbf{Performance across different privacy budgets and sample sizes.} Results for Dataset~1 over 10 runs with base values $\varepsilon=0.5$, $\delta=10^{-5}$, $n=3000$. The standard CIs only incorporate the sampling uncertainty (privacy considerations). $\Rightarrow$\,\emph{The results confirm our theoretical intuition: with larger budget or large $n$, our CIs approach the standard CIs.}}
    \label{fig:facet}
    \vspace{-0.4cm}
\end{figure}

\subsection{Evaluation on medical datasets}

We demonstrate the real-world applicability of \framework in a medical context on the Right Heart Catheterization (RHC) dataset \citep{Connors.1996}. The dataset contains observations of 5735 patients admitted to intensive care units. We aim to predict the ATE of receiving a Swan-Ganz catheterization on the probability of death following hospitalization. Further details on the RHC dataset are in Supplement~\ref{sec:appendix_experiments}.

\begin{table}[h]
\centering
\tiny
\begin{tabular}{cc r rr}
\toprule
Conf. & Budget & \multicolumn{1}{c}{DP-ATE} & \multicolumn{2}{c}{Confidence intervals} \\
\cmidrule(lr){4-5}
$1-\alpha$ & $\varepsilon$ & \multicolumn{1}{c}{$\hat{\tau}_\mathup{DP}$} & \multicolumn{1}{c}{Valid na\"ive CIs} & \multicolumn{1}{c}{\framework (Ours)}\\
\midrule
\multirow{3}{*}{0.80} 
& $0.10$ & $0.2398$ & $[-1.1619, 1.6416]$ & $[-1.1620, 1.6417]$\\
& $0.25$ & $0.0796$ & $[-0.4871, 0.6462]$ & $[-0.4871, 0.6463]$\\
& $0.50$ & $0.0541$ & $[-0.2361, 0.3442]$ & $[-0.2361, 0.3442]$\\
\midrule
\multirow{3}{*}{0.90} 
& $0.10$ & $0.2397$ & $[-1.5597, 2.0391]$ & $[-1.5599, 2.0393]$\\
& $0.25$ & $0.1059$ & $[-0.5885, 0.8003]$ & $[-0.5887, 0.8005]$\\
& $0.50$ & $0.0534$ & $[-0.3143, 0.4211]$ & $[-0.3143, 0.4210]$\\
\midrule
\multirow{3}{*}{0.95} 
& $0.10$ & $0.2517$ & $[-1.8680, 2.3715]$ & $[-1.8686, 2.3719]$\\
& $0.25$ & $0.0965$ & $[-0.7888, 0.9819]$ & $[-0.7891, 0.9822]$\\
& $0.50$ & $0.0468$ & $[-0.3856, 0.4792]$ & $[-0.3858, 0.4793]$\\
\bottomrule
\end{tabular}
\caption{\textbf{Medical application.} Reported are (i)~the standard \emph{non-private} CIs and (ii)~our \framework CIs on the RHC dataset averaged over 5 runs of the kernel base learner. $\Rightarrow$ \emph{As desired, the \framework CIs mainly include the valid but non-private CIs, which demonstrates that \framework achieves the desired empirical coverage, suggesting valid CIs.}
}
\label{tab:real_world}
\vspace{-0.4cm}
\end{table}

\textbf{Results:}
In Table~\ref{tab:real_world}, we present the CIs returned by \framework compared to an adapted version of the na\"ive CIs constructed with a \emph{valid variance}, which includes the correction term for the finite-sample variance stemming from ATE privatization. Again, we do not aim to benchmark the methods as the na\"ive CIs are non-private. Rather, we aim to understand whether the CIs from \framework are faithful. We observe that the \framework CIs include the valid, na\"ive CIs, which indicates faithfulness. In other words, our \framework gives CIs that remain informative while fulfilling $(\varepsilon, \delta)$-DP. 

\emph{Medical interpretation:} We observe that we always receive a positive private treatment effect $\hat{\tau}_\mathup{DP}$, indicating an average increase in the probability of death after receiving a Swan-Ganz catheter. The negative effect on survival is likely to be attributed to the life-threatening complications and risks of introducing the catheter, such as pulmonary artery rupture \citep[e.g.,][]{Kearney.1995}. The positive $\hat{\tau}_\mathup{DP}$ also coincides with the non-private ATE estimate $\hat{\tau} = 1.17$ and is in line with findings in the literature \citep[e.g.,][]{Connors.1996, Sandham.2003}.

\textbf{Extension to interventional RCT data:}
Our \framework framework applies to both observational and experimental data. In experimental randomized controlled trials (RCTs), the propensity score is known and thus does not need to be estimated. Note that in some cases, however, estimating the propensity score, even in RCTs, can be beneficial \citep{Su.2023}. So far, we have evaluated \framework on observational data. To show the applicability of \framework to RCTs, we conduct additional experiments in Supplement~\ref{sec:appendix_rct_extension}.

%%%%%%%%%%%%%%%%%%%%%%%%%%%%%%%%%%%%%%%%%%%%%%%%%%%%%%%%%%%%%%%%%%%%%%%%%%%%%%%
% Discussion
%%%%%%%%%%%%%%%%%%%%%%%%%%%%%%%%%%%%%%%%%%%%%%%%%%%%%%%%%%%%%%%%%%%%%%%%%%%%%%%

\section{Discussion}
\label{sec:discussion}

\textbf{Impact \& limitations:}
Our work has a direct impact on many practical applications, especially in safety-critical settings. Here, uncertainty quantification is necessary to make reliable and safe inferences. Yet, many applications are subject to strict privacy laws. We propose \framework to construct CIs for the ATE under DP. However, we recommend cautious use in practice as both ATE estimation and privacy mechanisms rely upon mathematical assumptions that need to be fulfilled. 

\textbf{Conclusion:}
We proposed a new, general framework for constructing differentially private CIs for the ATE. Our \framework framework is carefully designed to capture the additional uncertainty from the privatization and, therefore, return valid CIs. Our \framework framework is flexible: one can use \framework with any machine learning model for learning the propensity and the outcome function.

%%%%%%%%%%%%%%%%%%%%%%%%%%%%%%%%%%%%%%%%%%%%%%%%%%%%%%%%%%%%%%%%%%%%%%%%%%%%%%%
% Acknowledgements
%%%%%%%%%%%%%%%%%%%%%%%%%%%%%%%%%%%%%%%%%%%%%%%%%%%%%%%%%%%%%%%%%%%%%%%%%%%%%%%
\newpage
\section*{Acknowledgements}
This work has been supported by the German Federal Ministry of Education and Research (Grant: 01IS24082).

\bibliography{bibliography}

%%%%%%%%%%%%%%%%%%%%%%%%%%%%%%%%%%%%%%%%%%%%%%%%%%%%%%%%%%%%%%%%%%%%%%%%%%%%%%%
%%%%%%%%%%%%%%%%%%%%%%%%%%%%%%%%%%%%%%%%%%%%%%%%%%%%%%%%%%%%%%%%%%%%%%%%%%%%%%%
% APPENDIX
%%%%%%%%%%%%%%%%%%%%%%%%%%%%%%%%%%%%%%%%%%%%%%%%%%%%%%%%%%%%%%%%%%%%%%%%%%%%%%%
%%%%%%%%%%%%%%%%%%%%%%%%%%%%%%%%%%%%%%%%%%%%%%%%%%%%%%%%%%%%%%%%%%%%%%%%%%%%%%%
% \newpage

\makeatletter
\makeatother

\newcommand*\circledgreen[1]{%
\tikz[baseline=(char.base)]{
  \node[shape=circle, draw=ForestGreen!60, fill=ForestGreen!10, thick, inner sep=1pt] (char) {\scriptsize\textsf{#1}};
}}
\newcommand*\circledred[1]{%
\tikz[baseline=(char.base)]{
  \node[shape=circle, draw=BrickRed!60, fill=BrickRed!10, thick, inner sep=1pt] (char) {\scriptsize\textsf{#1}};
}}

\newcommand{\greencheck}{\textcolor{ForestGreen}{\checkmark}}
\newcommand{\redcross}{\textcolor{BrickRed}{\ding{55}}}

\makeatletter
\def\maketag@@@#1{\hbox{\m@th\normalfont\normalsize#1}}
\makeatother

%Appendix

\myexternaldocument{neurips_2025}

%title{Appendix}
%\maketitle
\appendix
\onecolumn

%\begin{document}

%%%%%%%%%%%%%%%%%%%%%%%%%%%%%%%%%%%%%%%%%%%%%%%%%%%%%%%%%%%%%%%%%%%%%%%%%%%%%%%
% Background
%%%%%%%%%%%%%%%%%%%%%%%%%%%%%%%%%%%%%%%%%%%%%%%%%%%%%%%%%%%%%%%%%%%%%%%%%%%%%%%

\section{Background}
\label{sec:appendix_background}

\subsection{Estimating the average treatment effect}
\label{sec:appendix_ATE_estimation}

Under the three standard assumptions of causal inference, i.e., positivity, consistency, and unconfoundedness \citep[e.g.,][]{Curth.2021, Rubin.2005}, the ATE is point-identified as 
\begin{align}
    \tau = \mathbb{E}[Y | A=1] - \mathbb{E}[Y | A=0] = \mathbb{E}[\mu(X,1) - \mu(X,0)].
\end{align}
However, estimating $\hat{\tau}$ via the empirical average of the stated difference is only unbiased if the treatment assignment is randomized and if it does not depend on $X$ (i.e., if for a constant $\pi(x)$). To account for a non-random treatment assignment in observational data, the \emph{inverse propensity weighted} (IPW) estimator constructs $\hat{\tau}$ by estimating $\hat{\pi}$ in the first step and then employing the estimates to take the propensity weighted empirical average of the outcomes \citep{Rosenbaum.1983}. 

Yet, misspecification of $\hat{\pi}$ or $\hat{\mu}$ can lead to biased ATE estimates. As a remedy, the \emph{augmented inverse propensity weighted} (AIPW) estimator \citep{Robins.1994} employs an orthogonal loss to fulfill the \emph{double robustness} property. Doubly robust estimators are consistent if \emph{either} $\hat{\mu}$ \emph{or} $\hat{\pi}$ are correctly specified. The AIPW estimator is then given by
\begin{align}
    \hat{\tau} = \frac{1}{n}\sum_{i=1}^n \Bigg[\hat{\mu}(X,1) - \hat{\mu}(X,0)
     + \frac{Y_i - \hat{\mu}(X,1)}{\hat{\pi}(X_i)}A_i - \frac{Y_i - \hat{\mu}(X,0)}{1 - \hat{\pi}(X_i)} (1- A_i) \Bigg].
\end{align}
Hence, in contrast to the IPW estimator, the AIPW estimator is an unbiased estimator of ATE and is asymptotically normal distributed \citep{Wager.2024}. This is important as we (i)~aim to consistently estimate the ATE, and (ii)~can derive valid CIs due to the asymptotic normality to the estimator. In our work, we thus aim to derive differentially private CIs for the ATE estimated through the AIPW estimator.

\subsection{Differential privacy}
\label{sec:appendix_DP}

In our work, we employ \emph{output perturbation} based on the Gaussian mechanism. Of note, output perturbation is especially suitable for our task of constructing CIs for the estimated ATE due to our three reasons: output perturbation (i)~retains the adherence to the causal assumptions, (ii)~is model agnostic, and (iii)~preserves the ability of the AIPW estimator to address the fundamental problem of causal inference \citep{Schroder.2025}. In contrast, other approaches (i.e., input, objective, and gradient perturbations) fail to fulfill at least one of the points (i)--(iii) and would thus be \underline{not} suitable for estimating the ATE with DP guarantees in a model-agnostic way.

Our work employs two important properties of differentially private algorithms, which we state below.

\begin{lemma}[Sequential composition property \citep{Dwork.2016}]
\label{lem:composition}
    Let the mechanism $M_j$, $j=1,\ldots, k$ satisfy $(\varepsilon_j, \delta_j)$-DP. Then, applying the sequence of the mechanisms $M_j$ on the same data yields $(\sum_{j=1}^{k}\varepsilon_j, \sum_{j=1}^{k}\delta_j)$-DP guarantees.
\end{lemma}

\begin{lemma}[Post-processing property \citep{Dwork.2006}]
\label{lem:post-processing}
    If a mechanism $M$ satisfies $(\varepsilon, \delta)$-DP, then, for any function $f$, it holds that $f(M)$ also satisfies $(\varepsilon, \delta)$-DP.
\end{lemma}

\subsection{Extended related work}
\label{sec:appendix_related_work}

\textbf{Uncertainty quantification for causal quantities:} Many different approaches for quantifying uncertainty in causal estimates exist in the literature. Many works rely on Bayesian methods \citep[e.g.,][]{alaa.2017, Hess.2024, jesson.2020, Horii.2024}. However, Bayesian methods require prior distributions informed by domain knowledge. Therefore, they are less robust to model misspecification and, therefore, not suitable for our model-agnostic framework. A different stream focuses on estimating the conditional distribution of the treatment effect, which then can be employed for bootstrapping to construct uncertainty intervals \citep{Ma.2024, Melnychuk.2023, Melnychuk.2024}. Another stream of literature provides finite-sample uncertainty guarantees based on conformal prediction methods \citep[e.g.,][]{lei.2021, Schroder.2024}. These methods can only construct intervals for the potential outcomes, not the treatment effects.
Although one could construct from conformal prediction intervals for treatment effects based on the potential outcomes, those intervals commonly tend to be very wide and, therefore, less informative in practice. Importantly, \underline{none} of the works above can provide $(\varepsilon, \delta)$-\emph{differentially private} confidence intervals, which is the novelty of our work.

\subsection{Doubly robust variance estimation}

Standard variance estimators for the AIPW estimator are, in general, only consistent if \emph{both} nuisance functions are correctly specified \citep{Gruber.2012}. Thus, these estimators are not doubly robust and, therefore, can easily lead to \emph{undercoverage} of the CIs in practice, a highly undesirable scenario in safety-critical applications. We thus aim to estimate $\hat{\sigma}_{\mathup{AIPW}}^2$ through a \emph{doubly robust} variance estimator.

A common such estimator is built upon nonparametric bootstrapping \citep{Funk.2011, ShookSa.2024}. However, bootstrapping requires many queries to the training data, which is highly disadvantageous for estimating the variance with DP guarantees. Further, bootstrapping is computationally expensive. This is unlike doubly robust variance estimators, which have been proposed only for specific applications, namely, survey data with missing data imputation \citep[e.g.,][]{Haziza.2012, Kim.2014}. 

In our work, we estimate $\hat{\sigma}_{\mathup{AIPW}}^2$ via the \emph{empirical sandwich estimator} \citep{Huber.1967}. This has the following benefits: (i)~the empirical sandwich estimator is doubly robust, and  (ii)~it consistently estimates the variance of the AIPW estimator \citep{ShookSa.2024}. 

Let the estimator $\hat{\theta}$ for $\theta_0$ on data $\{z_i \mid i=1,\ldots,n\}$ be the solution to 
\begin{align}
    \frac{1}{n}\sum_{i=1}^n \psi(z_i, \hat{\theta}) = 0.
\end{align}
Define $\mathbf{B}(\theta) := \mathbb{E}[-\frac{\partial}{\partial \theta}\psi(Z_i,\theta)]$ with $\mathbf{B}_n(\hat{\theta}) :=  \frac{1}{n}\sum_{i=1}^n[-\frac{\partial}{\partial \hat{\theta}}\psi(Z_i, \hat{\theta})]$ and $\mathbf{M}(\theta) := \mathbb{E}[\psi(Z_i, \theta)\psi(Z_i,\theta)^T]$ with $\mathbf{M}_n(\hat{\theta}) :=  \frac{1}{n}\sum_{i=1}^n[\psi(Z_i, \hat{\theta})\psi(Z_i, \hat{\theta})^T]$. Then, the variance of $\hat{\theta}$ can be estimated by the empirical sandwich variance estimator $\mathbf{V}_n(\hat{\theta})$ as
\begin{align}
    \mathbf{V}_n(\hat{\theta}) = \mathbf{B}_n(\hat{\theta})^{-1}\mathbf{M}_n(\hat{\theta})(\mathbf{B}_n(\hat{\theta})^{-1})^T.
\end{align}
For estimating the variance of the AIPW estimator, it suffices to assess the estimating equation $\psi_{\tau}(z_i, \hat{\tau}) = \Gamma_{\hat{\eta}}(z_i) - \hat{\tau}$. It is easy to see that $\mathbf{B}_n(\hat{\tau})^{-1} = -1$ and $\mathbf{M}_n(\hat{\tau})^{-1} = (\Gamma_{\hat{\eta}}(z_i) - \hat{\tau})^2$, yielding the estimator 
\begin{align}
    \mathbf{V}_n(\hat{\tau}) =  \frac{1}{n}\sum_{i=1}^n (\Gamma_{\hat{\eta}}(z_i) - \hat{\tau})^2 =  \frac{1}{n}\sum_{i=1}^n \mathrm{IF}^{\mathup{AIPW}}(z_i,\tau_0;P)^2
\end{align}
stated in Section~\ref{sec:variance_privatization}.

%%%%%%%%%%%%%%%%%%%%%%%%%%%%%%%%%%%%%%%%%%%%%%%%%%%%%%%%%%%%%%%%%%%%%%%%%%%%%%%
% Proofs
%%%%%%%%%%%%%%%%%%%%%%%%%%%%%%%%%%%%%%%%%%%%%%%%%%%%%%%%%%%%%%%%%%%%%%%%%%%%%%%

\newpage
\section{Proofs}
\label{sec:appendix_proofs}

\subsection{Proofs of the supporting lemmas}

\subsubsection{Proof of Lemma~\ref{lem:DR-IF}}
\textbf{Lemma 4.2.}
We make the following assumptions on the nuisance functions: The nuisance functions are (i)~bounded, (ii)~estimated at rates $n^{-\beta_{\mu}}$ and $n^{-\beta_{\pi}}$ with $\beta_{\mu} + \beta_{\pi} \geq \frac{1}{2}$, and (iii)~in a local neighborhood of the true nuisance functions, i.e., there exists $\lambda_n$ decreasing in the sample size $n$, s.t. $\rVert \hat{\eta}-\eta_0\lVert{\infty} \leq \lambda_n$. Then, the IF of the AIPW learner is dominated by the influence function of the final mean estimation in the second step. The influence stemming from training the nuisance estimators $\hat{\pi}$ and $\hat{\mu}$ is negligible in the privatization step.
\begin{proof}
    Our proof makes use of results for two-stage influence functions \citep{Zhelonkin.2012}. 

     Let a two-stage estimator with first- and second-stage score functions $\psi_1$ and $\psi_2$ and a function $h$ continuously piecewise differentiable in the second variable be defined as solving the equations
    \begin{align}
        \mathbb{E}_F[\psi_1(z^{(1)}, S(P))] = 0
        \qquad \text{and} \qquad
        \mathbb{E}_F[\psi_2(z^{(2)}, h(z^{(1)},S(P)), T(P))] = 0,
    \end{align}
    where $z_i = (z_i^{(1)}, z_i^{(2)})$ represents the data employed for first and second stage estimation, respectively, where $P$ denotes the distribution function of $z_i$, and where $S$ and $T$ denote the functional of the first and second stage estimation, respectively.
    Then, the influence function of the complete two-stage estimator is given by
    \begin{align} 
        \mathrm{IF}(z, T, P) = M^{-1}\bigg(&\psi_2(z^{(2)}, h(z^{(1)}, S(P)), T(P))  \\&+ \int \frac{\partial}{\partial \theta} \psi_2(\tilde{z}^{(2)},\theta, T(P))  \frac{\partial}{\partial \nu}h(\tilde{z}^{(1)}, \nu) \diff F (\tilde{z}) \cdot \mathrm{IF}(z, S, P) \bigg)
    \end{align}
    where $M = - \int \frac{\partial}{\partial\xi} \psi_2(\tilde{z}^{(2)}, h(\tilde{z}^{(1)}, S(P)), \xi) \diff F(\tilde{z})$ and $\mathrm{IF}(z, S, P)$ where denotes the IF of the first-stage estimation \citep{Zhelonkin.2012}. 

    For the AIPW estimator, we have $h(\cdot) = \Gamma_{\hat{\eta}}(z)$, $T(P) = \tau$, and, with a slight abuse of notation, $S(P) = (\mu, \pi)$. We aim to show that (i) $M=1$ and (ii) $\int \frac{\partial}{\partial \theta} \psi_2(\tilde{z}^{(2)},\theta, T(P)) \frac{\partial}{\partial \nu}h(\tilde{z}^{(1)}, \nu) \diff F (\tilde{z}) = 0$.

    To show (i), we observe that $\psi_2 = \frac{1}{n}\sum_{i=1}^n \Gamma_{\hat{\eta}}(z_i) - \tau$. Therefore, we directly get the desired equality $M=1$. To show (ii), we observe that, for the AIPW estimator, the integral equals $\mathbb{E}[\frac{\partial}{\partial \nu}h(\tilde{z}^{(1)}, \nu)] = \mathbb{E}[\frac{\partial}{\partial \eta} \Gamma_{\hat{\eta}}(z)]$, where $\eta = (\hat{\mu}, \hat{\pi})$. 
    Now, we observe that, due to the Neyman-orthogonality of the AIPW estimator \citep[e.g.,][]{Wager.2024} and the stability of the nuisance estimators, we have $\mathbb{E}[\frac{\partial}{\partial \eta} \Gamma_{\eta_0}(z)] = 0$.
    By Taylor expansion of $\Gamma_{\hat{\eta}}$ around $\Gamma_{\eta_0}$ it follows
    \begin{align}
        \Gamma_{\hat{\eta}}(z) = \Gamma_{\eta_0}(z) + \frac{\partial}{\partial \eta} \Gamma_{\eta_0}(z)(\hat{\eta}-\eta_0) + R_2,
    \end{align}
    where $R_2$ includes higher order terms, and thus 
    \begin{align}
        \mathbb{E}[\frac{\partial}{\partial \eta} \Gamma_{\hat{\eta}}(z)] = \mathbb{E}[\frac{\partial}{\partial \eta}R_2].
    \end{align}
    Recall that by the local neighborhood assumption there exists $\lambda_n$ small and decreasing in $n$, such that $\rVert \hat{\eta}-\eta_0\lVert{\infty} \leq \lambda_n$, and thus as well $R_2 << \lambda_n$. Note that $R_2$ depends on the Hessian matrix of $\Gamma_{\eta}$ with respect to the estimated nuisances. If the propensity is known, the Hessian vanishes by straightforward computation, and thus $\mathbb{E}[\frac{\partial}{\partial \eta} \Gamma_{\hat{\eta}}(z)] = 0$. If the propensity is unknown, we make use of the local neighborhood property in \citep{Chernozhukov.2018}, resulting in $\mathbb{E}[\frac{\partial}{\partial \eta} \Gamma_{\hat{\eta}}(z)] = \mathbb{E}[\frac{\partial}{\partial \eta}R_2]=0$. Overall, the IF of the AIPW estimator is thus dominated by the influence function of the final mean estimation in the second step.
\end{proof}

\subsubsection{Proof of Lemma~\ref{lem:IF_variance}}
\textbf{Lemma 4.4.}
The influence function of $\hat{\sigma}_{\mathup{AIPW}}^2$ is given by
\begin{align}
    \mathrm{IF}^{\sigma}(z, \sigma_0, P) = (\Gamma_{\hat{\eta}}(z) - \hat{\tau})^2 - \hat{\sigma}_{\mathup{AIPW}}^2.
\end{align}
\begin{proof}
    Recall that we estimate $\hat{\sigma}_{\mathup{AIPW}}^2$ by 
    \begin{align}
        \hat{\sigma}_{\mathup{AIPW}}^2 = \frac{1}{n}\sum_{i=1}^n  \mathrm{IF}^{\mathup{AIPW}}(z_i,\tau_0;P)^2
    \end{align}
    with $\mathrm{IF}^{\mathup{AIPW}}(z,\tau_0;P) = \Gamma_{\hat{\eta}}(z) - \hat{\tau}$. Therefore, $\hat{\sigma}_{\mathup{AIPW}}^2$ is estimated as the mean of the variable $(\Gamma_{\hat{\eta}}(z_i) - \hat{\tau})^2$. The result then directly follows with the IF of the mean.
\end{proof}

\subsection{Proofs of the main theorems}

\subsubsection{Proof of Theorem~\ref{thm:ate_privatization}}
\textbf{Theorem 4.3.}
Let $z:= (a,x,y)$ define a data sample following the joint distribution $\mathcal{Z}$ and $\hat{\eta} = (\hat{\mu},\hat{\pi})$ the estimated nuisance functions. Furthermore, let $D$ be the training dataset with $\vert D \vert$ = $n$. Define
\begin{align}
    \hat{\tau}_\mathup{DP} := \hat{\tau} + \sup_{z \in \mathcal{Z}} \big \lVert
    \Gamma_{\hat{\eta}}(z) - \hat{\tau} \big \rVert \cdot \frac{5\sqrt{2\ln(n)\ln{(2/\delta)}}}{\varepsilon n} \cdot U,
 \end{align}
where $U \sim \mathcal{N}(0, 1)$. Then, $\hat{\tau}_\mathup{DP}$ is $(\varepsilon, \delta)$-differentially private. 
\begin{proof}
    First, recall that the influence function of the AIPW estimator is given by $ \mathrm{IF}^{\mathup{AIPW}}(z,\tau_0;P) = \Gamma_{\hat{\eta}}(z) - \hat{\tau}$. Therefore, $\gamma_{\tau} := \sup_{z \in \mathcal{Z}} \big \lVert
    \Gamma_{\hat{\eta}}(z) - \hat{\tau} \big \rVert$ denotes the gross-error sensitivity of the AIPW estimator.

    Our proof builds upon the idea of upper-bounding the global sensitivity $\Delta_2(f)$ required for privatization with that Gaussian mechanism (Definition~\ref{def:gaussian_mechanism}) by the so-called \emph{smooth-sensitivity} \citep{Nissim.2007}. Specifically, for the $\xi$-smooth sensitivity $\mathit{SS}_{\xi}(f, D)$ of the estimator $f$ with $\xi = \frac{\varepsilon}{4(d+2\log(2/\delta))}$, it holds that
    \begin{align}
        f^{\mathup{DP}}_D = f_D + \frac{5\sqrt{2\log(2/\delta)}}{\varepsilon} \mathit{SS}_{\xi}(f, D)\cdot U, 
    \end{align}
    where $U \sim \mathcal{N}(0, 1)$ is $(\varepsilon, \delta)$-differential private with
    \begin{align}
        \mathit{SS}_{\xi}(f, D):= \sup_{D^{'}}\{ \exp(-\xi d_{\mathrm{H}}(D, D^{'})) LS(f, D^{'}) \mid  D^{'} \in \mathcal{Z}^n\},
    \end{align}
    where $\mathcal{Z}^n$ denotes the data domain and $LS(\cdot)$ is the local sensitivity given by
    \begin{align}
        \mathit{LS}(f, D):= \sup_{D^{'}} \left\{ \lVert f_D - f_{D^{'}}\rVert \;\Big|\; d_{\mathrm{H}}(D, D^{'})=1 \right\}
    \end{align}
    \citep[see][]{Nissim.2007}. However, similar to the global sensitivity, the smooth sensitivity is extremely expensive or even infeasible to calculate, depending on the machine learning model employed for nuisance estimation. Therefore, we follow \citet{AvellaMedina.2021} and upper bound the smooth sensitivity by the appropriately scaled gross-error sensitivity of the AIPW estimator.

    Observe that the AIPW estimate $\hat{\tau}$ solves the equation 
    \begin{align}
        \sum_{i=1}^n \mathrm{IF}^{\mathup{AIPW}}(z_i,\tau_0;P) = 0.
    \end{align}
    The influence function is differentiable with respect to $\hat{\tau}$ with derivative $\frac{\partial}{\partial \hat{\tau}} \mathrm{IF}^{\mathup{AIPW}}(z, \tau_0;P) = 1$. Furthermore, since we assume the data to stem from a bounded domain, there exists a constant $K \in \mathbb{R}$ such that, for the gross-error sensitivity of $\hat{\tau}$, one has $\gamma_{\tau} \leq K$. 
    As a result, we can upper bound the smooth sensitivity of the AIPW estimator by $\mathit{SS}_{\xi}(\hat{\tau}, D) \leq \frac{\sqrt{\log{(n)}}}{n} \gamma_{\tau}$ \citep{AvellaMedina.2021}, leading to the desired result in Theorem~\ref{thm:ate_privatization}.
\end{proof}

\subsubsection{Proof of Theorem~\ref{thm:variance_privatization}}
\textbf{Theorem 4.5.}     
Let $z:= (a,x,y)$ define a data sample following the joint distribution $\mathcal{Z}$, and let $\hat{\eta} = (\hat{\mu},\hat{\pi})$ be the estimated nuisance functions. Furthermore, let $D$ be the training dataset with $\vert D \vert$ = $n$.  For $U \sim \mathcal{N}(0, 1)$, we define
\begin{align}
    \hat{\sigma}^2_\mathup{DP} := \max \bigg\{0, \hat{\sigma}_{\mathup{AIPW}}^2 + \sup_{z \in \mathcal{Z}} \big \lVert
    \mathrm{IF}^{\sigma}(z, \sigma_0, P) \big \rVert \cdot \frac{5\sqrt{2\ln(n)\ln{(2/\delta)}}}{\varepsilon n} \cdot U \bigg\}.
\end{align}
\normalsize
\noindent
Then, $\hat{\sigma}^2_\mathup{DP}$ is $(\varepsilon, \delta)$-differentially private. 
\begin{proof}
    We proceed in the same manner as in the proof of Theorem~\ref{thm:ate_privatization}. We observe that $\hat{\sigma}_{\mathup{AIPW}}^2$ solves $\sum_{i=1}^n \mathrm{IF}^{\sigma}(z_i,\sigma_0;P) = 0$ with $\frac{\partial}{\partial \hat{\sigma}^2} \mathrm{IF}^{\sigma}(z, \sigma_0;P) = 1$. Again, since we assume the data to stem from a bounded domain, there exists a constant $K \in \mathbb{R}$ such that, for the gross-error sensitivity of $\hat{\sigma}^2$, we have $\gamma_{\sigma}:=  \sup_{z \in \mathcal{Z}} \big \lVert \mathrm{IF}^{\sigma}(z, \sigma_0, P) \big \rVert \leq K$. 
    Then, with Theorem 1 in \citet{AvellaMedina.2021}, we can upper bound the smooth sensitivity $\mathit{SS}_{\xi}(\hat{\sigma}^2, D)$ by $\frac{\sqrt{\log{(n)}}}{n} \gamma_{\sigma}$. As a result, we get that the quantity
    \begin{align}
        \hat{\sigma}_{\mathup{AIPW}}^2 + \sup_{z \in \mathcal{Z}} \big \lVert
    \mathrm{IF}^{\sigma}(z, \sigma_0, P) \big \rVert \cdot \frac{5\sqrt{2\ln(n)\ln{(2/\delta)}}}{\varepsilon n} \cdot U
    \end{align}
    is $(\varepsilon, \delta)$-private. 
    Now, observe that the maximum function does not require further queries to the underlying data. Therefore, by the post-processing theorem of DP, $\hat{\sigma}^2_\mathup{DP}$ is as well $(\varepsilon, \delta)$-private.
\end{proof}

\subsubsection{Proof of Theorem~\ref{thm:intervals}}
\textbf{Theorem 4.6.}
Let $z:= (a,x,y)$ define a data sample following the joint distribution $\mathcal{Z}$, and let $\hat{\eta} = (\hat{\mu},\hat{\pi})$ be the estimated nuisance functions. We define the privacy budget $(\varepsilon, \delta)$ and the sub-budgets $\varepsilon_1, \varepsilon_2, \delta_1, \delta_2$, for ATE and variance estimation under DP such that $\varepsilon_1 + \varepsilon_2 = \varepsilon$ and $\delta_1 + \delta_2 = \delta$. Furthermore, let $D$ be the training dataset with $\vert D \vert$ = $n$, and let $\hat{\tau}_{\mathup{DP}}$ denote the $(\varepsilon_1, \delta_1)$-differentially private ATE estimate and $\hat{\sigma}_{\mathup{DP}}$ denote the $(\varepsilon_2, \delta_2)$-differentially private variance estimate. Then, the valid and $(\varepsilon, \delta)$-differentially private CI is given by
\begin{align}
    \mathrm{CI}_{\mathup{DP}} := \Bigg[ \hat{\tau}^{\mathup{DP}} \pm \Phi^{-1}\Big(1-\frac{\alpha}{2}\Big) \sqrt{\frac{\hat{V}_{\mathup{DP}}}{n}}\Bigg],
\end{align}
where $\Phi^{-1}(1-\alpha/2)$ denotes the respective quantile of the standard normal distribution, $\gamma_{\tau}$ is the gross-error sensitivity of the ATE estimation, and
\begin{align}
    \hat{V}_{\mathup{DP}} =  \hat{\sigma}^2_{\mathup{DP}}+ \gamma_{\tau}^2 \cdot \frac{50\ln(n)}{n\varepsilon_1^{2}}\ln{\Big(\frac{2}{\delta_1}\Big)}.
\end{align}
\begin{proof}
    We need to show that (i) the provided CI is differentially private and (ii) that it is valid in the sense that the CI has asymptotic coverage at the targeted confidence level $1-\alpha$.

    First, we show (i):
    By Theorem~\ref{thm:ate_privatization} and Theorem~\ref{thm:variance_privatization}, $\hat{\tau}_{\mathup{DP}}$ and $\hat{\sigma}_{\mathup{DP}}$ are $(\varepsilon_1, \delta_1)$- and $(\varepsilon_2, \delta_2)$-differentially private, respectively. 
    As $\hat{V}_{\mathup{DP}}$ does not require further access to single data points, it as well is $(\varepsilon_2, \delta_2)$-differentially private by the post-processing property of DP (see Lemma~\ref{lem:post-processing}). As $\mathrm{CI}_{\mathup{DP}}$ states a non-data dependent combination of two privatized estimates, , $\mathrm{CI}_{\mathup{DP}}$ is $(\varepsilon_1 + \varepsilon_2, \delta_1 + \delta_2)$-differentially private by the parallel decomposition and the post-processing properties of DP (see Lemmas~\ref{lem:composition} and \ref{lem:post-processing}).

    Now, we turn to proving the validity of the CI (ii):
    Recall from Section~\ref{sec:intervals} that 
        \begin{align}
        \sqrt{n}\left(\hat{\tau}_{\mathup{DP}} - \tau_0\right) \rightarrow \sqrt{\sigma^2_{\mathup{AIPW}}} \cdot \mathcal{N} \left(0, 1\right).
    \end{align}
    As discussed in Section~\ref{sec:intervals}, we also aim to account for the finite sample variance stemming from privatization, to make the intervals not only asymptotically valid but also useful in finite sample settings. Thus, a valid CI which accounts for the finite sample privatization noise could be constructed around $\hat{\tau}_{\mathup{DP}}$ with variance $\sigma^2 = \sigma^2_{\mathup{AIPW}} + n \, r(\varepsilon, \delta, n)^2$. However, as $\sigma^2_{\mathup{AIPW}}$ is non-private, the resulting CI would also not be private. 
    
    Now observe that, by construction of $\sigma^2_{\mathrm{DP}}$, it also holds that $ \sqrt{n}\left(\sigma^2_{\mathup{DP}} - \sigma^2_{\mathup{AIPW}}\right) \rightarrow 0$, which proves the validity of $\mathrm{CI}_{\mathup{DP}}$. 
\end{proof}

%%%%%%%%%%%%%%%%%%%%%%%%%%%%%%%%%%%%%%%%%%%%%%%%%%%%%%%%%%%%%%%%%%%%%%%%%%%%%%%
% Experiments
%%%%%%%%%%%%%%%%%%%%%%%%%%%%%%%%%%%%%%%%%%%%%%%%%%%%%%%%%%%%%%%%%%%%%%%%%%%%%%%

\newpage
\section{Experiments}
\label{sec:appendix_experiments}

\subsection{Datasets}

\subsubsection{Synthetic data}
We generate our datasets similar to \citet{Oprescu.2019}:
\begin{align}
X_i &\sim \mathcal{U}[0, 1]^p \\
A_i &= \mathbf{1}\{(X^T\beta)_i \geq \eta_i\} \\
Y_i &= \tau A_i + X_i^T\gamma + \epsilon_i
\end{align}
where $\epsilon_i, \eta_i \sim \mathcal{U}[-1,1]$ are independent noise terms. The coefficients $\beta$ and $\gamma$ are sparse vectors with random support size $s$, where non-zero entries are drawn uniformly with $\beta_j \sim \mathcal{U}[0,0.3]$ and $\gamma_j \sim \mathcal{U}[0,1]$ for $j$ in the support. To satisfy the overlap assumption, propensity scores are clipped in the range $(0.1, 0.9)$, whereby the true average treatment effect $\tau$ is set to 1.0. Here, we consider two datasets: a low-dimensional dataset with $p=2$ covariates (Dataset~1) and a high-dimensional dataset with $p=24$ and $s=6$ (Dataset~2).

\subsubsection{Medical data}

The RHC dataset contains detailed real-world observations of 5735 patients admitted to intensive care units at different hospitals as part of the Study to Understand Prognoses and Preferences for Outcomes and Risks of Treatments (SUPPORT). We aim to predict the average treatment effect of receiving a Swan-Ganz catheterization, a diagnostic procedure performed on a pulmonary artery, on the probability of death in the 180 days after hospitalization. Our estimation is based on 8 confounders from medical practice (e.g., primary disease, heart rate, Glasgow coma score) and a binary treatment (Swan-Ganz catheterization in the first 24 hours after admission). We use a standard preprocessing pipeline that imputes missing values and removes outliers; our preprocessing pipeline is available in our GitHub repository.

\subsection{Implementation details}

All of our experiments are implemented in Python. We provide our code in our GitHub repository: 
\url{https://github.com/m-schroder/PrivATE}.
The experiments were run on an AMD Ryzen 7 PRO 6850U 2.70 GHz CPU with eight cores and 32GB RAM. The runtime varied for the different experiments (longer runtime for the more complex dataset 2). However, all experiments are easily feasible to compute on all standard computing resources within a short time. 

Our framework is highly flexible and can be used with any base machine learning method. We implement two versions: (1) a regression-based learner using logistic regression for the propensity score and kernel ridge regression (RBF kernel, $\alpha=0.1$) for the outcome model, and (2) a neural network-based learner with a simple multi-layer perceptron using one hidden layer of 32 neurons and $\tanh$ activation. The neural networks are optimized via stochastic gradient descent and regularized via L2-regularization ($\alpha=0.1$).

The \framework framework includes the calculation of a supremum over a bounded data space. We implemented this maximization problem using L-BFGS-B \citep{Byrd.1995}, a limited-memory Broyden–Fletcher–Goldfarb–Shanno algorithm for solving nonlinear optimization problems with bounded variables. We use default hyperparameters and use $n=10$ random starting points for every optimization.

By default, we choose $\varepsilon_1=\varepsilon_2$ as well as $\delta_1=\delta_2$ if not specified otherwise.

\newpage
\section{Trade-off between privacy and utility of the CIs}
\label{sec:appendix_tradeoff}

Introducing privacy to the ATE estimation and the CIs lowers the informativeness of the intervals due to the additional noise and thus wider intervals. However, this is expected and inherent to privatization, and thus \underline{not} specific to our framework. In practice, it might be interesting to visualize the privacy-utility tradeoff based on a user-defined utility function to decide on a level of privatization for a new analysis.

We performed such an analysis for our experiments, by comparing the utility of \framework against the utility of a bootstrapping method on the private ATE and the non-private estimation. Utility is defined as a weighted sum of privacy and interval width. We present the results for various privacy budgets in Fig.~\ref{fig:utility}. For all utility weightings, we observe that our method outperforms the baselines.

\begin{figure}[h]
    \centering
    \includegraphics[width=0.75\linewidth]{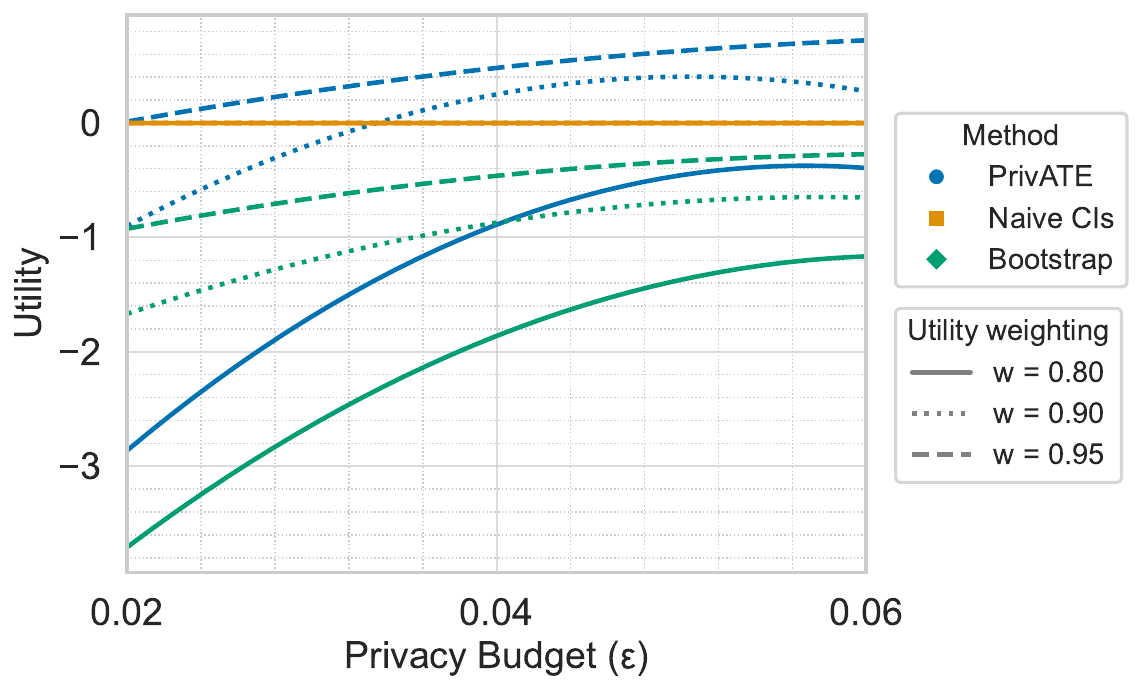}
    \caption{\textbf{Privacy-utility trade-off.} We report the utility curves with respect to the privacy budget for various utility functions (confidence level 0.95). As expected, we observe our \framework to outperform the naive and the bootstrapping method with increasing utility for increasing weight on the privacy constraint. The utility of the naive CIs does not vary significantly for different utility weightings.}
    \label{fig:utility}
\end{figure}

\newpage
\section{Extension to interventional RCT data}
\label{sec:appendix_rct_extension}

We now demonstrate the applicability of \framework to experimental data. In doing so, we can now additionally compare our method against the private differences of means method by \citet{Guha.2024}, yet which is restricted to experimental data (and thus \underline{not} applicable to the observational data used throughout our main paper). Again, we evaluate \framework using simulated and real-world RCT data.

\subsection{Dataset and implementation details}

Each sample consists of 1 uniformly distributed confounder, as in Dataset 1. To simulate an RCT setting, the treatment was then assigned by complete randomization with a factor of 0.5. The outcome generation again resembles the one from Dataset 1. To create the dataset, we sampled 1000 data samples from the described data-generating process. The implementation remained the same as for the other experiments.

\subsection{Results}

In Table~\ref{tab:RCT}, we present the average coverage and the corresponding standard deviation for both methods for confidence levels $\{0.80,0.90,0.95\}$ across ten runs. As expected, both methods roughly achieve the desired coverage.

\begin{table}[h]
\centering
\begin{tabular}{c ccc}
\toprule
Confidence & \multicolumn{3}{c}{Empirical coverage} \\
\cmidrule(lr){2-4}
$1-\alpha$ & \multicolumn{1}{c}{Difference of means} & \multicolumn{1}{c}{\framework (Kernel)} & \multicolumn{1}{c}{\framework (NN)}\\
\midrule
$0.80$ & $0.820 \pm 0.039$ & $0.780 \pm 0.042$ & $0.880  \pm 0.033$\\
\midrule
$0.90$ & $0.920 \pm 0.027$ & $0.900 \pm 0.030$ & $0.890 \pm 0.031$\\
\midrule
$0.95$ & $0.950 \pm 0.022$ & $0.940 \pm 0.024$ & $0.970 \pm 0.017$\\
\bottomrule
\end{tabular}
\vspace{0.2cm}
\caption{\textbf{Synthetic RCT evaluation.}}
\label{tab:RCT}
\end{table}
\vspace{-0.5cm}
When comparing the methods in terms of the widths of the intervals, we observe a significant difference (see Fig.~\ref{fig:RCT_synthetic}): The intervals given by the private differences of means method are much wider than the ones returned by our \framework and thus less informative. In other words: \emph{our method is preferred}.
\begin{figure}[h!]
    \vspace{-0.5cm}
    \centering
    \includegraphics[width=0.7\linewidth]{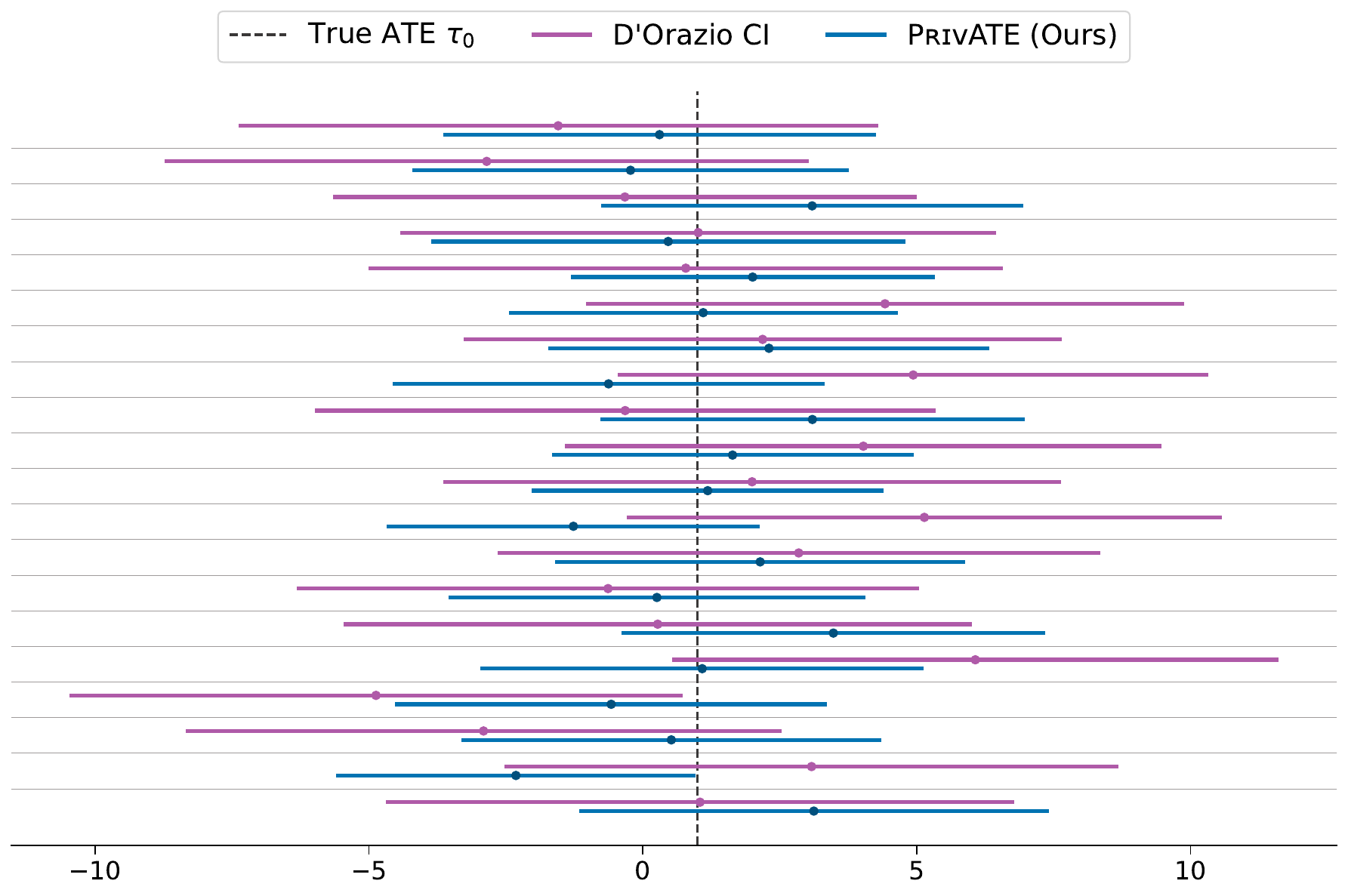}
    \caption{\textbf{Visualization of the CIs on RCT data.} The CIs are reported for both \framework and the difference in means method by \citet{Guha.2024}. For both, we use the NN base learner across different runs with privacy budgets $\varepsilon=0.5$ and $\delta=10^{-5}$. $\Rightarrow$ The intervals of the difference in means method are significantly wider than the ones of \framework. Our method returns more informative intervals and is thus preferred.}
    \label{fig:RCT_synthetic}
\end{figure}

%\end{document} 

\end{document}